\documentclass{article}

\usepackage[preprint]{neurips_2020}
\usepackage[utf8]{inputenc} 
\usepackage[T1]{fontenc}    
\usepackage{hyperref}       
\usepackage{url}            
\usepackage{booktabs}       
\usepackage{amsmath}
\usepackage{amsfonts}       
\usepackage{nicefrac}       
\usepackage{microtype}      
\usepackage{xcolor}
\usepackage{amsthm}
\usepackage{multirow}
\usepackage{enumitem}       
\usepackage{xspace}         

\usepackage{subcaption}
\usepackage[export]{adjustbox}
\usepackage{algorithm}
\usepackage{algpseudocode}
\usepackage{textcomp}

\newcommand{\bA}{ \mathbf{A }}
\newcommand{\ba}{ \mathbf{a }}
\newcommand{\bB}{ \mathbf{B }}
\newcommand{\bb}{ \mathbf{b }}

\newcommand{\bc}{ \mathbf{c }}
\newcommand{\bD}{ \mathbf{D }}

\newcommand{\bg}{ \mathbf{g }}
\newcommand{\bG}{ \mathbf{G }}

\newcommand{\bh}{ \mathbf{h }}
\newcommand{\bI}{ \mathbf{I }}
\newcommand{\bP}{ \mathbf{P }}

\newcommand{\bu}{ \mathbf{u }}

\newcommand{\bw}{ \mathbf{w }}

\newcommand{\bx}{ \mathbf{x }}

\newcommand{\bdelta}{\boldsymbol{\delta}}

\newcommand{\blambda}{\boldsymbol{\lambda}}

\newcommand{\bnu}{ \boldsymbol{\nu }}
\newcommand{\xbar}{ \mathbf{\bar{x}}}
\newcommand{\btA}{ \mathbf{\tilde{A}} }
\newcommand{\btb}{ \mathbf{\tilde{b}} }
\newcommand{\btG}{ \mathbf{\tilde{G}} }
\newcommand{\bth}{ \mathbf{\tilde{h}} }

\newcommand{\bone}{\mathbf{1}}
\newcommand{\bzero}{\mathbf{0}}

\newcommand{\cI}{\mathcal{I}}
\newcommand{\cJ}{\mathcal{J}}

\newcommand{\cK}{\mathcal{K}}
\newcommand{\cX}{\mathcal{X}}

\newcommand{\cW}{\mathcal{W}}
\newcommand{\bbR}{\mathbb{R}}

\newcommand{\xp}{\bx^*}
\newcommand{\xt}{\bx^\mathrm{obs}}

\newcommand{\SQPbprop}{\ensuremath{\text{SQP}_\mathrm{bprop}}\xspace}
\newcommand{\SQPimpl}{\ensuremath{\text{SQP}_\mathrm{impl}}\xspace}
\newcommand{\SQPdir}{\ensuremath{\text{SQP}_\mathrm{dir}}\xspace}
\newcommand{\SQPdiff}{\ensuremath{\text{SQP}_\mathrm{diff}}\xspace}

\renewcommand{\d}{\mathrm{d}}
\newcommand{\pd}[2][]{\frac{\partial#1}{\partial#2}}
\newcommand{\td}[2][]{\frac{\d#1}{\d#2}}
\newcommand{\grad}{\nabla}

\DeclareMathOperator*{\minimize}{\text{minimize}}
\DeclareMathOperator*{\maximize}{\text{maximize}}
\DeclareMathOperator*{\subjto}{\text{subject to}}

\DeclareMathOperator*{\argmin}{arg\,min}

\newtheorem{theorem}{Theorem}
\newtheorem{corollary}{Corollary}

\title{Learning Linear Programs from Optimal Decisions}

\author{%
  Yingcong Tan\\
  Concordia University\\
  Montreal, Canada\\
  \And
  Daria Terekhov \\
  Concordia University\\
  Montreal, Canada\\
  \And
  Andrew Delong\\
  Concordia University\\
  Montreal, Canada\\
}

\begin{document}

\maketitle

\begin{abstract}
  We propose a flexible gradient-based framework for learning linear programs from optimal decisions. 
  Linear programs are often specified by hand, using prior knowledge of relevant costs and constraints.
  In some applications, linear programs must instead be learned from observations of optimal decisions.
  Learning from optimal decisions is a particularly challenging bi-level problem, and much of the related {\em inverse optimization} literature is dedicated to special cases.
  We tackle the general problem, learning all parameters jointly while allowing flexible parametrizations of costs, constraints, and loss functions.
  We also address challenges specific to learning linear programs, such as empty feasible regions and non-unique optimal decisions.
  Experiments show that our method successfully learns synthetic linear programs and minimum-cost multi-commodity flow instances for which previous methods are not directly applicable.
  We also provide a fast batch-mode PyTorch implementation of the homogeneous interior point algorithm, which supports gradients by implicit differentiation or backpropagation.\vspace{-.3em}
\end{abstract}

\section{Introduction}\vspace{-.4em}
In linear programming, the goal is to make a optimal decision under a linear objective and subject to linear constraints. Traditionally, a linear program is designed using knowledge of relevant costs and constraints. More recently, methodologies that are data-driven have emerged. For example, in the ``predict-then-optimize'' paradigm \citep{Elmachtoub19}, linear programs are learned from direct observations of previous costs or constraints.

Inverse optimization~(IO) \citep{Burton92,Troutt95,Ahuja01}, in contrast, learns linear programs from observations of optimal decisions rather than of the costs or constraints themselves. The IO approach is particularly important when observations come from optimizing agents (e.g., experts \citep{Chan14,Barmann17} or customers \citep{Dong18}) who make near-optimal decisions with respect to their internal (unobserved) optimization models. 

From a machine learning perspective, the IO setup is as follows: we are given feature vectors $\{\bu_{1}, \bu_{2}, \dots, \bu_{N}\}$ representing conditions (e.g., time, prices, weather) and we observe the corresponding decision targets $\{\xt_{1}, \xt_{2}, \dots, \xt_{N}\}$ (e.g., quantities, actions) determined by an unknown optimization process, which in our case is assumed linear. We view IO as the problem of inferring a constrained optimization model that gives identical (or equivalent) decisions, and which generalizes to novel conditions $\bu$. The family of candidate models is assumed parametrized by some vector~$\bw$.

Learning a constrained optimizer that makes the observations both feasible \emph{and} optimal poses multiple challenges that have not been explicitly addressed. For instance, parameter setting $\bw_1$ in Figure~\ref{fig:example1} makes the observed decision $\xt_1$ optimal but not feasible, $\bw_2$ produces exactly the opposite result, and some $\bw$ values (black-hatched region in Figure~\ref{fig:example1}) are not even admissible because they will result in empty feasible regions. Finding a parameter such as $\bw_3$ that is consistent with the observations can be difficult. We formulate the learning problem in a novel way, and tackle it with gradient-based methods despite the inherent bi-level nature of learning. Using gradients from backpropagation or implicit differentiation, we successfully learn linear program instances of various sizes as well as learning the costs and right-hand coefficients of a minimum-cost multi-commodity flow problem.

\begin{figure}[t]
\hspace{-3.95em}\includegraphics[width=1.104\textwidth]{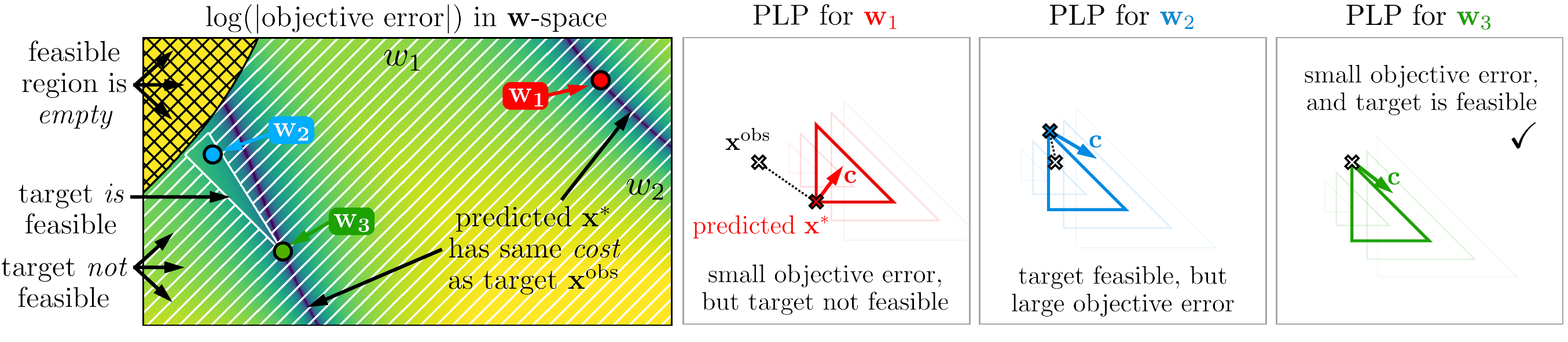}\vspace{-.35em}
\caption{A depiction of our constrained learning formulation. We learn a parametric linear program (PLP), here parametrized by a feature $u$ and weights $\bw \!=\! (w_1, w_2)$ and using a single training observation $(u_1, \xt_1)$. The PLP corresponding to three parameter settings $\bw_1, \bw_2, \bw_3$ are shown, with the cost vector and feasible region corresponding to $u_1$ emphasized. The goal of learning is to find solutions such as $\bw^* = \bw_3$. (See Appendix for the specific PLP used in this example.) \label{fig:example1} \vspace{-1em}}
\end{figure}

{\bf Parametric Linear Programs \;} In a linear program (LP), decision variables $\bx \in \bbR^D$ may vary, and the cost coefficients ${\bc \in \bbR^D}$, inequality constraint coefficients ${\bA \in \bbR^{M_1 \times D}}$, ${\bb \in \bbR^{M_1}}$, and equality constraint coefficients ${\bG \in \bbR^{M_2 \times D}}$, ${\bh \in \bbR^{M_2}}$ are all constant. In a {\em parametric linear program} (PLP), the coefficients (and therefore the optimal decisions) may depend on features~$\bu$. In order to infer a PLP from data, one may define a suitable hypothesis space parametrized by $\bw$. We refer to this hypothesis space as the form of our {\em forward optimization problem} (FOP).\\[-0.4em] 
\noindent\begin{minipage}{.25\linewidth}
\begin{equation} \label{eq:LP} \tag{LP}
\begin{aligned}
\textstyle \min_{\bx} \;\; & \bc^T \bx \\
\text{s.t.}           \;\; & \bA \bx \leq \bb \\
                           & \bG \bx =    \bh
\end{aligned}
\end{equation}
\end{minipage}
\begin{minipage}{.34\linewidth}
\begin{equation} \label{eq:PLP} \tag{PLP}
\begin{aligned}
\textstyle \min_{\bx} \;\; & \bc(\bu)^T \bx \\
\text{s.t.}           \;\; & \bA(\bu)\bx \leq \bb(\bu)\\
                           & \bG(\bu)\bx =    \bh(\bu)
\end{aligned}
\end{equation}
\end{minipage}
\begin{minipage}{.41\linewidth}
\begin{equation} \label{eq:FOP} \tag{FOP}
\begin{aligned}
\textstyle \min_{\bx} \;\; & \bc(\bu, \bw)^T \bx \\
\text{s.t.}           \;\; & \bA(\bu, \bw)\bx \leq \bb(\bu, \bw)\\
                           & \bG(\bu, \bw)\bx =    \bh(\bu, \bw)
\end{aligned}
\end{equation}
\end{minipage}\\[.4em]
A choice of hypothesis $\bw$ in \eqref{eq:FOP} identifies a PLP, and a subsequent choice of conditions $\bu$ identifies an LP. The LP can then be solved to yield an optimal decision $\bx^*$ under the model. These predictions of optimal decisions can be compared to observations at training time, or can be used to anticipate optimal decisions under novel conditions $\bu$ at test time. 
\vspace{-.5em}
\section{Related Work}\label{sec:related_work} \vspace{-.3em}

{\bf Inverse optimization \;}
IO has focused on developing optimization models for minimally adjusting a prior estimate of $\bc$ to make a single feasible observation $\xt$ optimal \citep{Ahuja01,Heuberger04} or for making $\xt$ minimally sub-optimal to~\eqref{eq:LP} without a prior $\bc$ \citep{Chan14,Chan19}. Recent work \citep{Babier19} develops exact approaches for imputing non-parametric $\bc$ given multiple potentially infeasible solutions to \eqref{eq:LP}, and to finding non-parametric $\bA$ and/or $\bb$ \citep{Chan18c,Ghobadi20}. In the parametric setting, joint estimation of $\bA$ and $\bc$ via a maximum likelihood approach was developed by \citet{Troutt05,Troutt08} when only $\bh$ is a function of $\bu$.  \citet{Gallego17} jointly learn $\bc$ and $\bb$ which are affine functions of $\bu$. \citet{Barmann17,Barmann20} and \cite{Dong18} study online versions of inverse linear and convex optimization, respectively, learning a sequence of cost functions where the feasible set for each observation are assumed to be fully-specified. \citet{tan2019dio} proposed a gradient-based approach for learning cost and constraints of a PLP, by `unrolling' a barrier interior point solver and backpropagating through it. Their formulation does not aim to avoid situations where a training target is infeasible, like the one shown in Figure~\ref{fig:example1} for $\bw_1$.  

In inverse convex optimization, the focus has been in imputing parametric cost functions while assuming that the feasible region is known for each~$\bu_i$ \citep{Keshavarz11,Bertsimas15,Aswani18, Esfahani18}, usually under assumptions of a convex set of admissible~$\bu$, the objective and/or constraints being convex in $\bu$, and uniqueness of the optimal solution for every~$\bu$. Furthermore, since the feasible region is fixed for each~$\bu$, it is simply assumed to be non-empty and bounded, unlike for our work. Although our work focuses on linear programming, it is otherwise substantially more general, allowing for learning of all cost and constraint coefficients simultaneously with no convexity assumptions related to~$\bu$, no restrictions on the existence of multiple optima, and explicit handling of empty or unbounded feasible regions. 

{\bf Optimization task-based learning \;}
\citet{Kao09} introduces the concept of directed regression, where the goal is to fit a linear regression model while minimizing the decision loss, calculated with respect to an unconstrained quadratic optimization model. \citet{Donti2017} use a neural network approach to minimize a task loss which is calculated as a function of the optimal decisions in the context of stochastic programming. \citet{Elmachtoub19} propose the ``Smart Predict-then-Optimize'' framework in which the goal is to predict the cost coefficients of a linear program with a fixed feasible region given past observations of features and true costs, i.e., given $(\bu_i, \bc_i)$. Note that knowing $\bc_i$ in this case implies we can solve for~$\bx^*_i$, so our framework can in principle be applied in their setting but not vice versa. Our framework is still amenable to more `direct' data-driven prior knowledge: if in addition to $(\bu_i, \xp_i)$ we have partial or complete observations of $\bc_i$ or of constraint coefficients, regressing to these targets can easily be incorporated into our overall learning objective.

{\bf Structured prediction \;} In structured output prediction \citep{Taskar05,Bakir07,Nowozin14,Daume15}, each prediction is $\bx^* \in \argmin_{\bx \in \cX(\bu)} f(\bx, \bu, \bw)$ for an objective $f$ and known output structure $\cX(\bu)$. In our work the structure is also learned, parametrized as $\mathcal{X}(\bu, \bw) = \left\{\, \bx \mid \bA(\bu, \bw) \bx \leq \bb(\bu, \bw), \, \bG(\bu, \bw) \bx = \bh(\bu, \bw) \,\right\}$, and the objective is linear $f(\bx, \bu, \bw) = \bc(\bu, \bw)^T \bx$. In structured prediction the loss $\ell$ is typically a function of $\xp$ and a target $\xbar$, whereas in our setting it is important to consider a parametric loss $\ell(\bx^*, \xbar, \bu, \bw)$.

{\bf Differentiating through an optimization \;}
Our work involves differentiating through an LP. \citet{bengio2000gradient} proposed gradient-based tuning of neural network hyperparameters and, in a special case, backpropagating through the Cholesky decomposition computed during training (suggested by L\'eo Bottou). \citet{Stoyanov11} proposed backpropagating through a truncated loopy belief propagation procedure. \citet{domke2012generic,Domke13} proposed automatic differentiation through truncated optimization procedures more generally, and \citet{maclaurin2015gradient} proposed a similar approach for hyperparameter search. 
The continuity and differentiability of the optimal solution set of a quadratic program has been extensively studied~\citep{lee2006quadratic}. \citet{Amos2017} recently proposed integrating a quadratic optimization layer in a deep neural network, and used implicit differentiation to derive a procedure for computing parameter gradients. As part of our work we specialize their approach, providing an expression for LPs. Even more general is recent work on differentiating through convex cone programs \citep{agrawal2019differentiating}, submodular optimization \citep{djolonga2017differentiable}, and arbitrary constrained optimization~\citep{gould2019deep}. There are also versatile perturbation-based differentiation techniques \citep{papandreou2011perturb,berthet2020learning}.

\vspace{-.1em}
\section{Methodology}\label{sec:methodology}\vspace{-.1em}
Here we introduce our new bi-level formulation and methodology for learning parametric linear programs. Unlike previous approaches (e.g. 
\citet{Aswani18}), we do not transform the problem to a single-level formulation, and so we do not require simplifying assumptions. We propose a technique for tackling our bi-level formulation with gradient-based non-linear programming methods.

\vspace{-.5em}
\subsection{Inverse Optimization as PLP Model Fitting} \label{sec:io} \vspace{-.25em}

Let $\{(\bu_i, \xt_i)\}_{i=1}^N$ denote the training set. A loss function $\ell(\xp, \xt, \bu, \bw)$ penalizes discrepancy between prediction $\xp$ and target $\xt$ under conditions $\bu$ for the PLP hypothesis identified by~$\bw$. Note that if $\xt_i$ is optimal under conditions $\bu_i$, then $\xt_i$ must also be feasible.
We therefore propose the following bi-level formulation of the {\em inverse linear optimization problem} (ILOP):\\[-.8em]
\begin{subequations}
\newcommand{\subequationsformat}{\theparentequation.\arabic{equation}}
\begin{align}\label{eq:ILOP}\tag{ILOP}
\minimize_{\bw \in \cW} \quad &  \textstyle \frac{1}{N} \sum_{i=1}^{N} \ell(\bx^*_i, \xt_i, \bu_i, \bw) + r(\bw)\\[.0em]
\text{subject to}         \quad & \bA(\bu_i, \bw) \xt_i \leq \bb(\bu_i, \bw) , \quad\! \bG(\bu_i, \bw) \xt_i =    \bh(\bu_i, \bw), & i = 1,\ldots,N \label{eqn:IOP_outer}\\[-.0em]
&\bx^*_i \in \argmin_\bx \left\{\: \bc(\bu_i, \bw)^T\bx \;\bigg|\; 
                      \begin{aligned}
                       \bA(\bu_i, \bw)\bx \leq \bb(\bu_i, \bw)\\
                       \bG(\bu_i, \bw)\bx =    \bh(\bu_i, \bw)
                      \end{aligned} \: \right\}, & i = 1,\ldots,N \label{eqn:IOP_inner}
\end{align}
\end{subequations}\\[-.3em]
\noindent where $r(\bw)$ simply denotes an optional regularization term such as $r(\bw) = \| \bw \|^2$ and $\cW \subseteq \bbR^{K}$ denotes additional problem-specific prior knowledge, if applicable (similar constraints are standard in the IO literature \citep{Keshavarz11,Chan19}). The `inner' problem \eqref{eqn:IOP_inner} generates predictions $\xp_i$ by solving $N$ independent LPs. The `outer' problem tries to make these predictions consistent with the targets $\xp_i$ while also satisfying target feasibility (\ref{eqn:IOP_outer}).

Difficulties may arise, in principle and in practice. An inner LP may be infeasible or unbounded for certain $\bw \in \cW$, making $\ell$ undefined. Even if all $\bw \in \cW$ produce feasible and bounded LPs, an algorithm for solving~\eqref{eq:ILOP} may still attempt to query $\bw \notin \cW$. The outer problem as a whole may be subject to local minima due to non-convex objective and/or constraints, depending on the problem-specific parametrizations. We propose gradient-based techniques for the outer problem (Section~\ref{sec:algorithms}), but $\td[\ell]{\bw}$ may not exist or may be non-unique at certain $\bu_i$ and $\bw$ (Section~\ref{sec:iograd}). 

Nonetheless, we find that tackling this formulation leads to practical algorithms. To the best of our knowledge, \eqref{eq:ILOP} is the most general formulation of inverse linear parametric programming. It subsumes the non-parametric cases that have received much interest in the IO literature. 

{\bf Choice of loss function \;} The IO literature considers {\em decision error}, which penalizes difference in decision variables, and {\em objective error}, which penalizes difference in optimal objective value~\citep{Babier19}. A fundamental issue with decision error, such as {\em squared decision error} (SDE) $\ell(\xp, \xt) = \frac{1}{2}\| \xp_i - \xt_i \|^2$, is that when $\xp$ is non-unique the loss is also not unique; this issue was also a motivation for the ``Smart Predict-then-Optimize'' paper \citep{Elmachtoub19}. An objective error, such as {\em absolute objective error} (AOE) $\ell(\xp, \xt, \bc) = | \bc^T(\xt_i - \xp_i) |$, is unique even if $\bx^*$ is not. We evaluate AOE using imputed cost $\bc(\bu, \bw)$ during training; this usually requires at least some prior knowledge $\cW$ to avoid trivial cost vectors, as in \citet{Keshavarz11}.

{\bf Target feasibility \;} Constraints~\eqref{eqn:IOP_outer} explicitly enforce target feasibility $\bA \xt_i \leq \bb,\, \bG \xt_i = \bh$ in any learned PLP. 
The importance of these constraints can be understood through Figure~\ref{fig:example1}, where hypothesis $\bw_1$ achieves $\text{AOE}\!=\!0$ since $\xt$ and $\xp$ are on the same hyperplane, despite $\xt$ being infeasible. \citet{Chan19} show that if the feasible region is bounded then for any infeasible $\xt$ there exists a cost vector achieving $\text{AOE}\!=\!0$.

{\bf Unbounded or infeasible subproblems \;} Despite~\eqref{eqn:IOP_outer}, an algorithm for solving~\eqref{eq:ILOP} may query a $\bw$ for which an LP in~\eqref{eqn:IOP_inner} is itself infeasible and/or unbounded, in which case a finite $\xp$ is not defined. We can extend \eqref{eq:ILOP} to explicitly account for these special cases (by penalizing a measure of infeasibility~\citep{murty2000infeasibility}, and penalizing unbounded directions when detected) but in our experiments simply evaluating the (large) loss for an arbitrary $\bx^*$ returned by our interior point solver worked nearly as well at avoiding such regions of $\cW$, so we opt to keep the formulation simple.

{\bf Noisy observations \;}
Formulation \eqref{eq:ILOP} can be extended to handle measurement noise. For example, individually penalized non-negative slack variables can be added to the right-hand sides of~\eqref{eqn:IOP_outer} as in a soft-margin SVM \citep{cortes1995support}. Alternatively, a norm-penalized group of slack variables can be added to each $\xt_i$ on the left-hand side of~\eqref{eqn:IOP_outer}, softening targets in decision space. We leave investigation of noisy data and model-misspecification as future work.

\subsection{Learning Linear Programs with Sequential Quadratic Programming} \label{sec:algorithms}

We treat \eqref{eq:ILOP} as a {\em non-linear programming} (NLP) problem, making as few assumptions as possible. We focus on {\em sequential quadratic programming} (SQP), which aims to solve NLP problems iteratively. Given current iterate $\bw^{k}$, SQP determines a search direction $\bdelta^k$ and then selects the next iterate $\bw^{k+1} = \bw^k + \alpha \bdelta^k$ via line search on $\alpha > 0$. Direction $\bdelta^k$ is the solution to a quadratic program.
\begin{equation*} \label{eq:nlp}
\begin{aligned}
\textstyle \minimize_{\bw} \;\; & f(\bw)  &\quad& & 
\textstyle \minimize_{\bdelta} \;\; & \grad f(\bw^{k})^T \bdelta + \bdelta^T \bB^{k}\bdelta \\[-.1em]
\subjto         \;\; & \bg(\bw) \leq \bzero \quad\; \text{(NLP)}&& &
\subjto         \;\; & \grad \bg(\bw^k)^T \bdelta + \bg(\bw^k) \leq \bzero  \quad\; \text{(SQP)}\\[-.1em]
                      & \bh(\bw) =    \bzero && &
                      & \grad \bh(\bw^k)^T \bdelta + \bh(\bw^k) =    \bzero
\end{aligned}
\end{equation*}
Each instance of subproblem (SQP) requires evaluating constraints\footnote{NLP constraint vector $\bh(\bw)$ is not the same as FOP right-hand side $\bh(\bu, \bw)$, despite same symbol.} and their gradients at $\bw^k$, as well as the gradient of the objective. Matrix $\bB^k$ approximates the Hessian of the Lagrange function for (NLP), where $\bB^{k+1}$ is typically determined from the gradients by a BFGS-like update. Our experiments use an efficient variant called {\em sequential least squares programming} (SLSQP) \citep{schittkowski1982nonlinear2,kraft1988software} which exploits a stable $LDL$ factorization of $\bB$.

The NLP formulation of~\eqref{eq:ILOP} has $N M_1$ inequality and $N M_2$ equality constraints from \eqref{eqn:IOP_outer}:
\begin{equation*}
\begin{aligned}
\bg(\bw) = \begin{bmatrix}
  \bA(\bu_i, \bw) \xt_i - \bb(\bu_i, \bw)
  \end{bmatrix}_{i=1}^{M_1},\quad
&& 
\bh(\bw) = \begin{bmatrix}
  \bG(\bu_i, \bw) \xt_i - \bh(\bu_i, \bw)
  \end{bmatrix}_{i=1}^{M_2}.
\end{aligned}
\end{equation*}
plus any constraints needed to enforce $\bw \in \cW$. The NLP constraint residuals and their gradients $\grad \bg(\bw),\grad \bh(\bw)$ can be directly evaluated. Evaluating $f(\bw) = \frac{1}{N}\sum_{i=1}^{N} \ell(\bx^*_i, \xt_i\!, \bu_i, \bw) + r(\bw)$ requires solving each LP in~\eqref{eqn:IOP_inner}. Finally, evaluating $\grad f(\bw)$ requires evaluating vector-Jacobian product $\td[\ell]{\bw} = \pd[\ell]{\bw} + \pd[\ell]{\xp_i}\pd[\xp_i]{\bw}$ for each $i$, which requires differentiating through the LP optimization that produced $\xp_i$ from $\bu_i$ and $\bw$. That is exactly what we do, and this approach allows us to tackle~\eqref{eq:ILOP} directly in its bi-level form, using powerful gradient-based NLP optimizers like SQP as the `outer' solver. Section~\ref{sec:iograd} compares methods for the differentiating through LP optimization. 

{\bf Redundant NLP constraints \;} When PLP model parameters $\bw$ have fixed dimension, the NLP formulation of \eqref{eq:ILOP} can involve many redundant constraints, roughly in proportion to $N$. Indeed, if $\cW \subseteq \bbR^K$ and $K < N M_2$ the equality constraints may appear to over-determine $\bw$, treating (NLP) as a feasibility problem; but, due to redundancy $\bw$ is not uniquely determined. The ease or difficulty of removing redundant constraints from (NLP) depends on the domain-specific parametrizations of PLP constraints $\bA(\bu, \bw), \bb(\bu, \bw), \bG(\bu, \bw),$ and $\bh(\bu, \bw)$. Equality constraints that are affinely-dependent on $\bw$ can be eliminated from (NLP) by a simple pseudoinverse technique, resulting in a lower-dimensional problem; this also handles the case where (NLP) is not strictly feasible in $\bh(\bw) = \bzero$ (either due to noisy observations or model misspecification) by automatically searching only among $\bw$ that exactly minimize the sum of squared residuals $\|\bh(\bw)\|^2$. If equality constraints are polynomially-dependent on $\bw$, we can eliminate redundancy by Gr\"obner basis techniques~\citep{cox2013ideals} although, unlike the affine case, it may not be possible or beneficial to reparametrize-out the new non-redundant basis constraints from the NLP. Redundant inequality constraints can be either trivial or costly to identify~\citep{telgen1983identifying}, but are not problematic. See Appendix for details.

{\bf Benefit over gradient-free methods \;} Evaluating $f(\bw)$ is expensive in our NLP because it requires solving $N$ linear programs. To understand why access to $\grad f(\bw)$ is important in this scenario, it helps to contrast SQP with a well-known gradient-free NLP optimizer such as COBYLA~\citep{powell1994direct}. For $K$-dimensional NLP, COBYLA maintains $K+1$ samples of $f(\bw), \bg(\bw), \bh(\bw)$ and uses them as a finite-difference approximation to $\grad f(\bw^k), \grad g(\bw^k), \grad h(\bw^k)$ where $\bw^k$ is the current iterate (best sample). The next iterate $\bw^{k+1}$ is computed by optimizing over a trust region centered at~$\bw^k$. COBYLA recycles past samples to effectively estimate `coarse' gradients, whereas SQP uses gradients directly. Figure~\ref{fig:example1algorithms} shows SLSQP and COBYLA running on the example from Figure~\ref{fig:example1}.

\begin{figure}[t]
\hspace{-2.3em}\includegraphics[width=.98\textwidth]{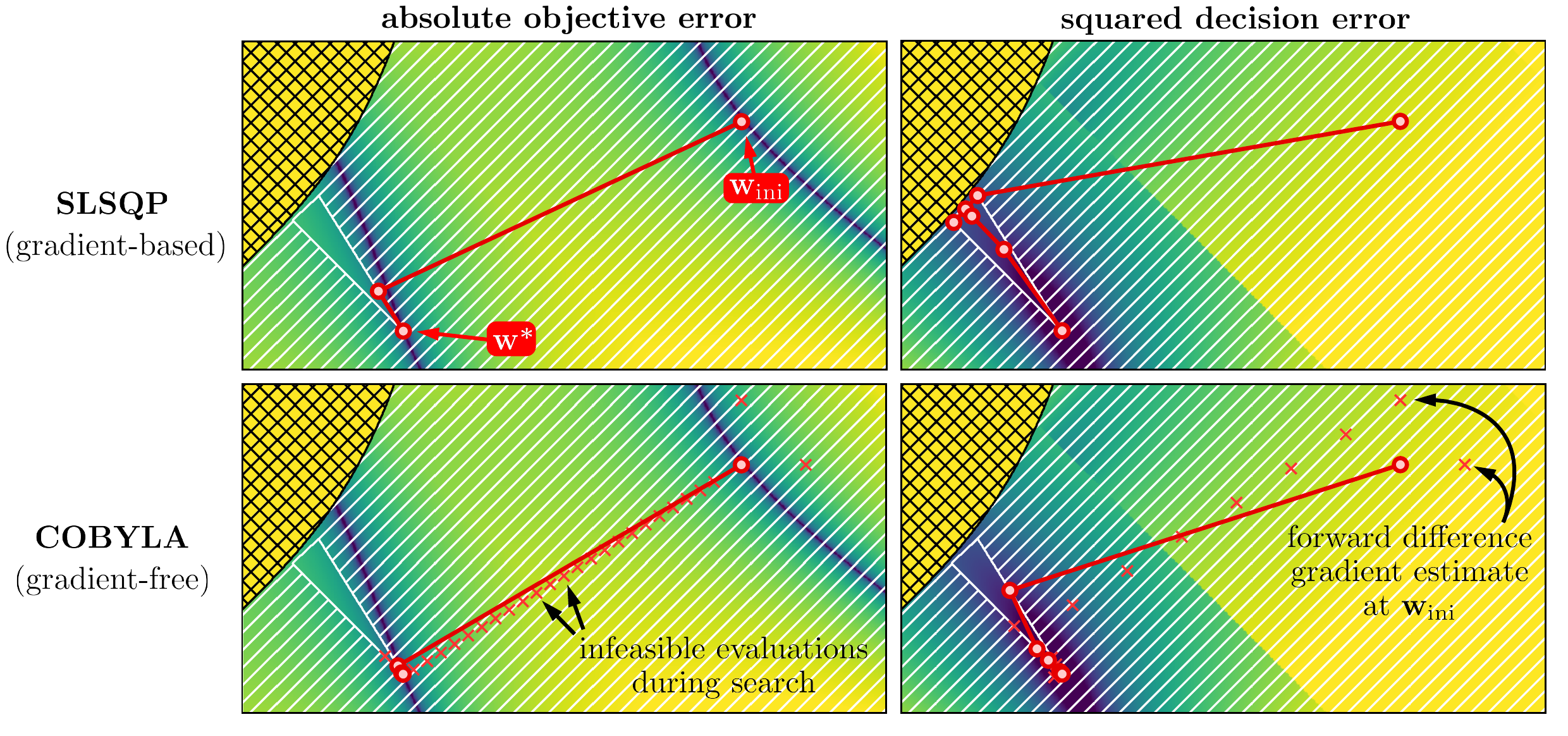}\vspace{-.5em}
\caption{An illustration of how SLSQP and COBYLA solve the simple learning problem in Figure~\ref{fig:example1} for the AOE and SDE loss functions. Each algorithm first tries to satisfy the NLP constraints $\bg(\bw) \leq \bzero$ (triangle-shaped feasible region in $\bw$-space), then makes progress minimizing $f(\bw)$. \label{fig:example1algorithms} \vspace{-1em}}
\end{figure}


\vspace{-.3em}
\subsection{Computing Loss Function Gradients}\vspace{-.3em}\label{sec:iograd}
If, at a particular point $(\bu_i, \bw)$, each corresponding vector-Jacobian product $\pd[\ell]{\xp_i}\pd[\xp_i]{\bw}$ exists, is unique, and can be computed, then we can construct~(SQP) at each step. For convenience, we assume that $(\bc, \bA, \bb, \bG, \bh)$ are expressed in terms of $(\bu, \bw)$ within an automatic differentiation framework such as PyTorch, so all that remains is to compute Jacobians $(\pd[\ell]{\bc}, \pd[\ell]{\bA}, \pd[\ell]{\bb}, \pd[\ell]{\bG}, \pd[\ell]{\bh})$ at each $(\bu_i, \bw)$ as an intermediate step at the outset of backpropagation. We consider three approaches: 
\begin{enumerate}[itemsep=-.05em,topsep=-.18em,leftmargin=1.75cm]
\item[{\em backprop:}]{backpropagate through the steps of the homogeneous interior point algorithm for LPs,}
\item[{\em implicit:}]{specialize the implicit differentiation procedure of~\citet{Amos2017} to LPs, and}
\item[{\em direct:}]{evaluate gradients directly, in closed form (for objective error only).}
\end{enumerate}

We implemented a batch PyTorch version of the homogeneous interior point algorithm~\citep{andersen2000mosek,xu1996simplified} developed for the MOSEK optimization suite and currently the default linear programming solver in SciPy~\citep{2020SciPy}. Our implementation is also efficient in the backward pass, for example re-using the $LU$ decomposition\footnote{Cholesky decomposition is also supported and re-used, but we use $LU$ decomposition in experiments.} from each Newton step.

For implicit differentiation we follow~\citet{Amos2017} by forming the system of linear equations that result from differentiating the KKT conditions and then inverting that system to compute the needed vector-Jacobian products. For LPs this system can be poorly conditioned, especially at strict tolerances on the LP solver, but in practice it provides useful gradients.

For direct evaluation (in the case of objective error), we use Theorem~\ref{thm:oegrad}. 
When $\ell$ is AOE loss, by chain rule we can multiply each quantity by $\pd[\ell]{z} = \mathrm{sign}(z)$ to get the needed Jacobians.

\begin{theorem}  \label{thm:oegrad}
Let $\bx^* \in \bbR^D$ be an optimal solution to \eqref{eq:LP} and let $\blambda^* \in \bbR^{M_1}_{\leq 0}, \bnu^* \in \bbR^{M_2}$ be an optimal solution to the associated dual linear program. If $\bx^*$ is non-degenerate then the objective error $z = \bc^T (\xt - \bx^*)$ is differentiable and the total derivatives\footnote{In slight abuse of notation, we ignore leading singleton dimension of $\pd[z]{\bA} \in \bbR^{1 \times M_1 \times D}, \pd[z]{\bG} \in \bbR^{1 \times M_2 \times D}$.} are
\begin{equation*}
\begin{aligned}\\[-2.0em]
       \textstyle \pd[z]{\bc} &= \left(\xt - \bx^*\right)^T
&\quad \textstyle \pd[z]{\bA} &=  \blambda^*\bx^{*T}
&\quad \textstyle \pd[z]{\bb} &= -\blambda^{*T}
&\quad \textstyle \pd[z]{\bG} &= \bnu^* \bx^{*T}
&\quad \textstyle \pd[z]{\bh} &= -\bnu^{*T}.\\[-.6em]
\end{aligned}
\end{equation*}
\end{theorem}
Gradients $\pd[z]{\bb}$ and $\pd[z]{\bh}$ for the right-hand sides are already well-known as {\em shadow prices}. If $\bx^*$ is degenerate then the relationship between shadow prices and dual variables breaks down, resulting in two-sided shadow prices \citep{Strum69,Aucamp82}.

We use degeneracy in the sense of \citet{tijssen1998balinski}, where a point on the relative interior of the optimal face need not be degenerate, even if there exists a degenerate vertex on the optimal face. This matters when $\bx^*$ is non-unique because interior point methods typically converge to the analytical center of the relative interior of the optimal face~\citep{Zhang94}. Tijssen and Sierskma also give relations between degeneracy of $\bx^*$ and uniqueness of $\blambda^*, \bnu^*$, which we apply in Corollary~\ref{thm:oeunique}. When the gradients are non-unique, this corresponds to the subdifferentiable case.

\begin{corollary} \label{thm:oeunique}
In Theorem~\ref{thm:oegrad}, both $\pd[z]{\bb}$ and $\pd[z]{\bh}$ are unique,
$\pd[z]{\bc}$ is unique if and only if $\bx^*$ is unique, and 
both $\pd[z]{\bA}$ and $\pd[z]{\bG}$ are unique if and only if $\bx^*$ is unique or $\bc = \bzero$.
\end{corollary}

\vspace{-.5em}
\section{Experiments}\label{sec:experiments}\vspace{-.3em}

We evaluate our approach by learning a range of synthetic LPs and parametric instances of minimum-cost multi-commodity flow.
Use of synthetic instances is common in IO (e.g., \citet{Ahuja01, Keshavarz11, Dong18}) and there are no community-established and readily-available benchmarks, especially for more general formulations.
Our experimental study considers instances not directly addressable by previous IO work, either because we learn all coefficients jointly or because the parametrization results in non-convex NLP. 

We compare three versions\footnote{For completeness we also evaluated finite-differences (\SQPdiff) which, unsurprisingly, was not competitive.} of our gradient-based method (\SQPbprop, \SQPimpl, \SQPdir) with two gradient-free methods: random search (RS) and COBYLA. The main observation is that the gradient-based methods perform similarly and become superior to gradient-free methods as the dimension~$K$ of parametrization $\bw$ increases.  We find that including a black-box baseline like COBYLA is important for assessing the practical difficulty of an IO instance (and encourage future papers to do so) because such methods work reasonably well in low-dimensional problems. A second observation is that generalization to testing conditions is difficult because the discontinuous nature of LP decision space creates an underfitting phenomenon. This may explain why many previous works in IO require a surprising amount of training data for so few model parameters (see end of Section~\ref{sec:experiments}). A third observation is that there are instances for which no method succeeds at minimizing training error 100\% of the time. Our method can therefore be viewed as a way to significantly boost the probability of successful training, when combined with simple global optimization strategies such as multi-start.

Experiments used PyTorch~v1.6 nightly build, the COBYLA and SLSQP wrappers from SciPy~v1.4.1, and were run on an Intel Core i7 with 16GB RAM. (We do not use GPUs, though our PyTorch interior point solver inherits GPU acceleration.) We do not regularize $\bw$ nor have any other hyperparameters.

\begin{figure}[t]
\centering
    \begin{subfigure}{\textwidth}
      \centering
          \includegraphics[width=\textwidth]{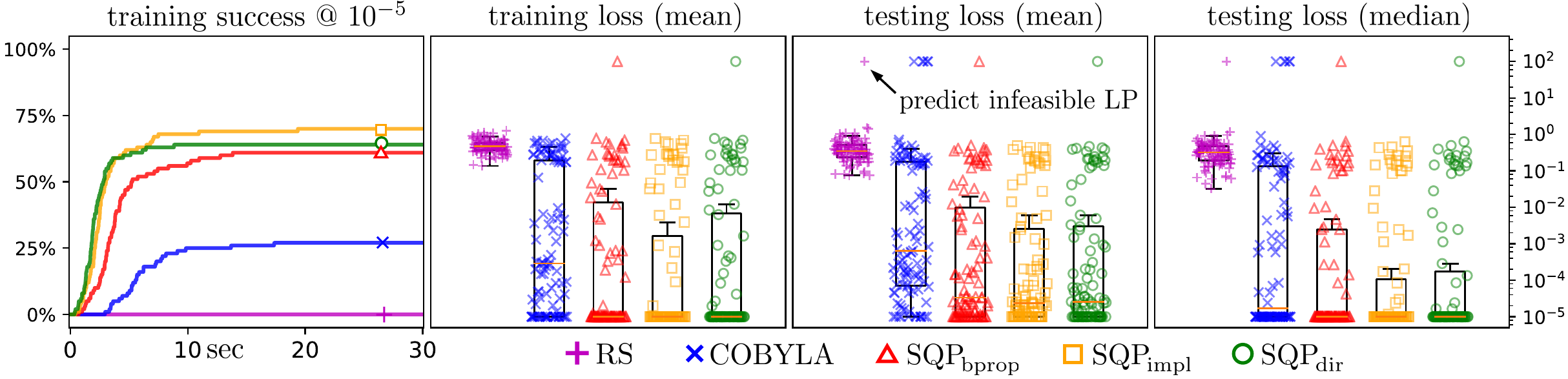}\vspace{-.5em}
    \end{subfigure} 
\caption{A comparison on synthetic PLP instances. Shown is the probability of achieving zero AOE training loss over time (curves), along with final training and testing loss (box plots). Each mark denotes one of 100 trials (different instances) with  20 training and testing points ($D\!=\!10, M_1\!=\!80$). The AOE testing loss is always evaluated with the `true' cost $\bc$, never the imputed cost. For insight into why the mean testing error is larger than median testing error, see discussion (end of Section~\ref{sec:experiments}).%
\label{fig:exp_1a}\vspace{-.5em}}%
\end{figure}

{\bf Learning linear programs \;}  We used the LP generator of~\citet{tan2019dio}, modifying it to create a more challenging variety of feasible regions; their code did not perform competitively in terms of runtime or success rate on these harder instances, and cannot be applied to AOE loss.
Fig.~\ref{fig:exp_1a} shows the task of learning ($\bc$, $\bA$, $\bb$) with a $K\!=\!6$ dimensional parametrization $\bw$, a $D\!=\!10$ dimensional decision space $\bx$, and 20 training observations. RS fails; COBYLA `succeeds' on ~25\% of instances; SQP succeeds on 60-75\%, which is substantially better. The success curve of \SQPbprop slightly lags those of \SQPimpl and \SQPdir due to the overhead of backpropagating through the steps of the interior point solver. See Appendix for five additional problem sizes, where overall the conclusions are the same.  
Surprisingly, $\SQPimpl$ works slightly better than $\SQPbprop$ and $\SQPdir$ in problems with higher $D$. We observe similar performance on instances with equality constraints, where $\bG$ and $\bh$ also need to be learned (see Appendix). Note that each RS trial returns the best of (typically) thousands of $\bw$ settings evaluated during the time budget, all sampled uniformly from the same $\cW$ from which the `true' synthetic PLP was sampled. Most random (and thus initial) points do not satisfy~\eqref{eqn:IOP_outer}.

Learning $(\bc, \bA, \bb)$ directly, so that $\bw$ comprises all LP coefficients, results in a high-dimensional NLP problem (which is why, to date, the IO literature has focused on special cases of this problem, either with a single $\xt$ \citep{Chan18,Chan19} or fewer coefficients to learn \citep{Ghobadi20}). For example, an instance with $D\!=\!10, M_1\!=\!80$ has $890$ adjustable parameters. \SQPbprop, \SQPimpl and \SQPdir consistently achieve zero AOE training loss,  while RS and COBYLA consistently fail to make learning progress given the same time budget (see Appendix). 

\begin{figure}[b]
\includegraphics[width=\textwidth]{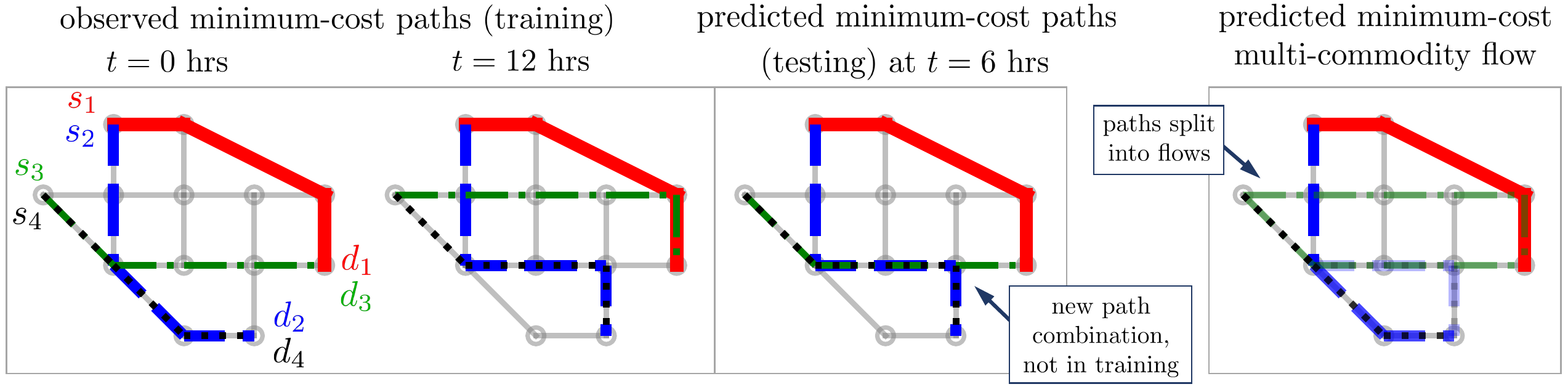}\vspace{-.5em}
\caption{A visualization of minimum-cost paths (for simplicity) and minimum-cost multi-commodity flows (our experiment) on the Nguyen-Dupuis network. Sources $\{s_1, s_2, s_3, s_4\}$ and destinations $\{d_1, d_2, d_3, d_4\}$ are shown. At left are two example sets of training paths $\{(t_1, \xt_1), (t_2, \xt_2)\}$ alongside an example of a correctly predicted set of optimal paths under different conditions (different~$t$). At right is a visualization of a correctly predicted optimal flow, where color intensity indicates proportion of flow along arcs.  \label{fig:nguyen}}
\end{figure}

\begin{figure}[t]
\centering
    \begin{subfigure}{\textwidth}
      \centering
          \includegraphics[width=\linewidth]{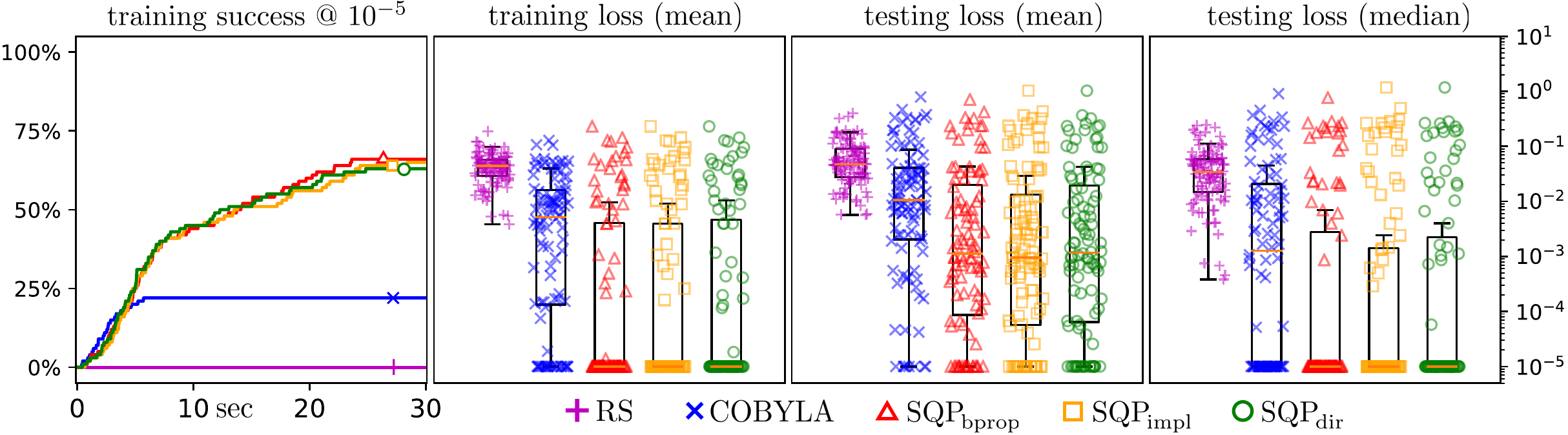}
    \end{subfigure} \vspace{-.4em} 
\caption{A comparison on minimum-cost multi-commodity flow instances, similar to Fig.~\ref{fig:exp_1a}. \label{fig:multicommodityResults}}
\end{figure}

{\bf Learning minimum-cost multi-commodity flow problems \;} 
Fig.~\ref{fig:nguyen} shows a visualization of our experiment on the Nguyen-Dupuis graph \citep{Nguyen84}. 
We learn a periodic arc cost $c_j(t, l_j, p_j) = l_j +  w_1 p_j + w_2l_j(\sin(2\pi(w_3+w_4t+w_5l_j))+1)$ and an affine arc capacity $b_j(l_j) = 1+ w_6 + w_7l_j$, based on global feature $t$ (time of day) and arc-specific features $l_j$ (length) and $p_j$ (toll price). To avoid trivial solutions, we set $\cW = \{\bw \ge \mathbf{0}, w_3+w_4+w_5 = 1\}$. 
Results on 100 instances are shown in Fig.~\ref{fig:multicommodityResults}.
The SQP methods outperform RS and COBYLA in training and testing loss. From an IO perspective the fact that we are jointly learning costs and capacities  in a non-convex NLP formulation is already quite general. Again, for higher-dimensional parametrizations, we can expect the advantage of gradient-based methods to get stronger. 

We report both the mean and median loss over the testing points in each trial. The difference in mean and median testing error is due to the presence of a few `outliers' among the test set errors. Fig.~\ref{fig:decisionsurface} shows the nature of this failure to generalize: the decision map $\bu \mapsto \xp$ of a PLP has discontinuities, so the training data can easily under-specify the set of learned models that can achieve zero training loss, similar to the scenario that motivates max-margin learning in SVMs. It is not clear what forms of regularization $r(\bw)$ will reliably improve generalization in IO. Fig.~\ref{fig:decisionsurface} also suggests that training points which closely straddle discontinuities are much more `valuable' from a learning perspective.
\begin{figure}
\hspace{-.5em}\includegraphics[width=1\textwidth]{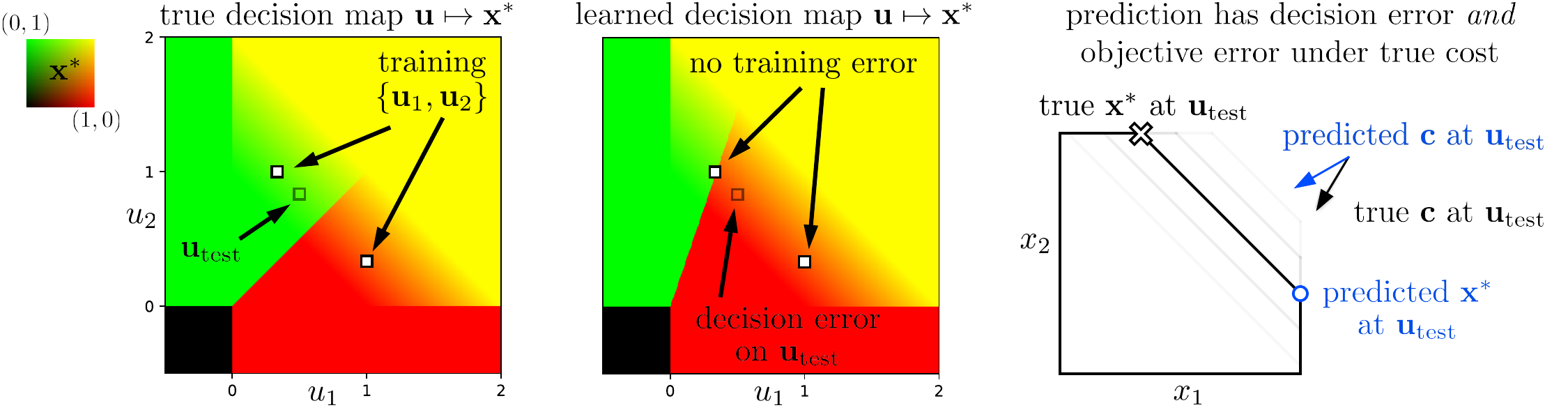}\\[-2em]
\small \phantom{~~~~~~~~~} (a) \hspace{1.33in} (b) \hspace{1.4in} (c)\vspace{.25em}
\caption{A failure to generalize in a learned PLP. Shown are the optimal decision map $\bu \mapsto \bx^*$ for a ground-truth PLP (a) and learned PLP (b) with the value of components $(\textcolor[rgb]{0.85,0,0}{x_1^*}, \textcolor[rgb]{0,0.7,0}{x_2^*})$ represented by red and green intensity respectively, along with that of a PLP trained on $\{\bu_1, \bu_2\}$. The learned PLP has no training error ($\text{SOE}\!=\!0, \text{AOE}\!=\!0$) but large test error ($\text{SOE}\!=\!.89, \text{AOE}\!=\!.22$) as depicted in~(c). (See Appendix for the specific PLP used in this example.) \label{fig:decisionsurface}\vspace{-1em}}
\end{figure}

\vspace{-.4em}
\section{Conclusion}\label{sec:conclusion} \vspace{-.35em}
In this paper, we propose a novel bi-level formulation and gradient-based framework for learning linear programs from optimal decisions. The methodology learns all parameters jointly while allowing flexible parametrizations of costs, constraints, and loss functions---a generalization of the problems typically addressed in the inverse linear optimization literature. 

Our work facilitates a strong class of inductive priors, namely parametric linear programs, to be imposed on a hypothesis space for learning. A major motivation for ours and for similar works is that, when the inductive prior is suited to the problem, we can learn a much better (and more interpretable) model, from far less data, than by applying general-purpose machine learning methods. In settings spanning economics, commerce, and healthcare, data on decisions is expensive to obtain and to collect, so we hope that our approach will help to build better models and to make better decisions.

\bibliography{File/bibliography.bib}
\clearpage
\section*{Appendix}\label{sec:appendix}
\subsection*{Appendix A: Forward Optimization Problem for Figure 1}\label{sec:appendixA}
{\em Forward optimization problem for Figure~\ref{fig:example1}.}
The FOP formulation used is shown in~\eqref{eq:example1_plp} below. 
\begin{equation} \label{eq:example1_plp}
\begin{alignedat}{1}
\minimize_{x_1, x_2} \quad & \cos(w_1 + w_2 u) x_1 + \sin(w_1 + w_2 u) x_2 \\
\subjto   \quad & (1 + w_2 u) x_1 \ge w_1 \\
                & (1 + w_1) x_2   \ge w_2 u \\
                & x_1 + x_2       \le 1 + w_1 + w_2 u
\end{alignedat}
\end{equation}
For a fixed $u$ and weights $\bw=(w_1, w_2)$ it is an LP. The observation $\xt_1=(-0.625,0.925)$ was generated using $u_1=1.0$ with true parameters $\bw=(-0.5,-0.2)$.

For illustrative clarity, the panels in Figure~\ref{fig:example1} depicting the specific feasible regions for $\{\bw_1, \bw_2, \bw_3\}$ are slightly adjusted and stylized from the actual PLP~\eqref{eq:example1_plp}, but are qualitatively representative.

\subsection*{Appendix B: Redundancy Among Target-Feasibility Constraints}

Redundant constraints in \eqref{eqn:IOP_outer} are not problematic in principle.  Still, removing redundant constraints may help overall performance, either in terms of speed or numerical stability of the `outer' solver. Here we discuss strategies for automatically removing redundant constraints, depending on assumptions. In this section, when we use $\bx$ or $\bx_i$ it should be understood to represent some target $\xt$ or $\xt_i$.

{\bf Constraints that are equivalent. \;} There may exist indices $i$ and $i'$ for which the corresponding constraints $\ba(\bu_i, \bw)^T\bx_i \leq b(\bu_i, \bw)$ and $\ba(\bu_{i'}, \bw)^T\bx_{i'}\leq b(\bu_{i'}, \bw)$ are identical or equivalent. For example, when a constraint is independent of $\bu$ this often results in identical training targets $\bx_i$ and $\bx_{i'}$ that produce identical constraints. The situation for equality constraints is similar.

{\bf Constraints independent of $\bw$. \;} If an individual constraint $\ba(\bu, \bw)^T \bx \leq b(\bu, \bw)$ is independent of $\bw$ then either:
\begin{enumerate}[itemsep=-.05em,topsep=-.18em,leftmargin=0.7cm]
    \item $\ba(\bu_i)^T \bx_i \leq b(\bu_i)$ for all $i$ so the constraint can be omitted; or,
    \item $\ba(\bu_i)^T \bx_i > b(\bu_i)$ for some $i$ so the \eqref{eq:ILOP} formulation is infeasible due to model misspecification, either in structural assumptions, or assumptions about noise.
\end{enumerate}

The same follows for any equality constraint $\bg(\bu, \bw)^T \bx = h(\bu, \bw)$ that is independent of $\bw$. For example, in our minimum-cost multi-commodity flow experiments, the flow conservation constraints (equality) are independent of $\bw$ and so are omitted from \eqref{eqn:IOP_outer} in the corresponding ILOP formulation.

{\bf Constraints affinely-dependent in $\bw$. \;} Constraints may be affinely-dependent on parameters $\bw$. For example, this is a common assumption in robust optimization \citep{zhen2018adjustable}. Let $\bA(\bu, \bw)$ and $\bb(\bu, \bw)$ represent the constraints that are affinely dependent on $\bw \in \bbR^K$. We can write
\begin{equation*}
\begin{aligned}
\bA(\bu, \bw) = \bA^{\!0}(\bu) + \sum_{k=1}^K w_k \bA^{\!k}(\bu) & \text{\qquad and \quad} &
\bb(\bu, \bw) = \bb^0(\bu) + \sum_{k=1}^K w_k \bb^k(\bu)
\end{aligned}
\end{equation*}
for some matrix-valued functions $\bA^{k}(\cdot)$ and vector-valued functions $\bb^k(\cdot)$. It is easy to show that we can then rewrite the constraints $\bA(\bu, \bw) \bx \leq \bb(\bu, \bw)$ as $\btA(\bu, \bx) \bw \leq \btb(\bu, \bx)$ where
\begin{equation*}
\begin{aligned}
\btA(\bu, \bx) &= \begin{bmatrix} \bA^{\!1}(\bu) \bx - \bb^1(\bu) & \cdots & \bA^{\!K}\!(\bu) \bx - \bb^{K}(\bu) \end{bmatrix} \\
\btb(\bu, \bx) &= \bb^0(\bu) - \bA^{\!0}(\bu) \bx.
\end{aligned}
\end{equation*}
Similarly if $\bG(\bu, \bw) \bx = \bh(\bu, \bw)$ are affine in $\bw$ we can rewrite them as $\btG(\bu, \bx) \bw = \bth(\bu, \bx)$. If we apply these functions across all training samples $i=1,\ldots,N$, and stack their coefficients as
\begin{equation*}
\btA = \begin{bmatrix}
\btA(\bu_i, \bx_i)
\end{bmatrix}_{i=1}^N, \quad 
\btb = \begin{bmatrix}
\btb(\bu_i, \bx_i) \\
\end{bmatrix}_{i=1}^N, \quad
\btG = \begin{bmatrix}
\btG(\bu_i, \bx_i)
\end{bmatrix}_{i=1}^N, \quad 
\bth = \begin{bmatrix}
\bth(\bu_i, \bx_i)
\end{bmatrix}_{i=1}^N
\end{equation*}
then the corresponding ILOP constraints \eqref{eqn:IOP_outer} reduce to a set of linear `outer' constraints $\btA \bw \leq \btb$ and $\btG \bw = \bth$ where $\btA \in \bbR^{N M_1 \times K}, \btb \in \bbR^{N M_1}, \btG \in \bbR^{N M_2 \times K}, \bth \in \bbR^{N M_2}$. These reformulated constraint matrices are the system within which we eliminate redundancy in the affinely-dependent case, continued below.

{\bf Equality constraints affinely-dependent in $\bw$. \;} We can eliminate affinely-dependent equality constraint sets by reparametrizing the ILOP search over a lower-dimensional space; this is what we do for the experiments with equality constraints shown in Figure~\ref{fig:1b}, although the conclusions do not change with or without this reparametrization. To reparametrize the ILOP problem, compute a Moore-Penrose pseudoinverse $\btG^+ \in \bbR^{K \times NM_2}$ to get a direct parametrization of constrained vector $\bw$ in terms of an unconstrained vector $\bw' \in \bbR^{K}$:
\begin{equation} \label{eq:pseudoinverse}
\bw(\bw') = \btG^+\bth + (\bI - \btG^+ \btG) \bw'.
\end{equation}
By reparametrizing \eqref{eq:ILOP} in terms of $\bw'$ we guarantee $\btG \bw(\bw') = \bth$ is satisfied and can drop equality constraints from \eqref{eqn:IOP_outer} entirely. There are three practical issues with~\eqref{eq:pseudoinverse}:
\begin{enumerate}[itemsep=-.05em,topsep=-.18em,leftmargin=0.7cm]
    \item Constrained vector $\bw$ only has $K' \equiv K - \mathrm{rank}(\btG)$ degrees of freedom, so we would like to re-parametrize over a lower-dimensional $\bw' \in \bbR^{K'}$.
    \item To search over $\bw' \in \bbR^{K'}$ we need to specify $\btA' \in \bbR^{N M_1 \times K'}$ and $\btb' \in \bbR^{N M_1}$ such that $\btA' \bw' \leq \btb'$ is equivalent to $\btA \bw(\bw') \leq \btb$.
    \item Given initial $\bw_\mathrm{ini} \in \bbR^K$ we need a corresponding $\bw'_\mathrm{ini} \in \bbR^{K'}$ to initialize our search.
\end{enumerate}

To address the first issue, we can let the final $K-K'$ components of $\bw' \in \bbR^{K}$ in \eqref{eq:pseudoinverse} be zero, which corresponds to using a lower-dimensional $\bw' \in \bbR^{K'}$. As shorthand let matrix $\bP \in \bbR^{K \times K'}$ be
\begin{equation*}
\begin{aligned}
\bP &\equiv (\bI_{K \times K} - \btG^+ \btG) \bI_{K \times K'} = \bI_{K \times K'} - (\btG^+ \btG)_{1:K,1:K'}
\end{aligned}
\end{equation*}
where $\bI_{K \times K'}$ denotes $\begin{bmatrix} \,\bI_{K' \times K'} \\ \bzero_{(K-K') \times K'} \end{bmatrix}$ as in \texttt{torch.eye(K, K')} and $(\bG^+ \bG)_{1:K,1:K'}$ denotes the first $K'$ columns of $K \times K$ matrix $\bG^+ \bG$. Then we have $\bw(\bw') = \bG^+\bh + \bP \bw'$ where the full dimension of $\bw' \in \bbR^{K'}$ matches the degrees of freedom in $\bw$ subject to $\btG \bw = \bth$ and we have $\btG \bw(\bw') = \bth$ for any choice of $\bw'$.

To address the second issue, simplifying $\btA \bw(\bw') \leq \btb$ gives inequality constraints $\btA' \bw' \leq \btb'$ with $\btA' = \btA \bP$ and $\btb' = \btb - \btA \btG^+ \bth$.

To address the third issue we must solve for $\bw'_\mathrm{ini} \in \bbR^{K'}$ in the linear system $\bP \bw'_\mathrm{ini} = \bw_\mathrm{ini} - \btG^+\bth$. Since $\mathrm{rank}(\bP) = K'$ the solution exists and is unique.

Consider also the effect of this reparametrization when $\btG \bw = \bth$ is an infeasible system, for example due to noisy observations or misspecified constraints. In that case searching over $\bw'$ automatically restricts the search to $\bw$ that satisfy $\btG \bw = \bth$ in a least squares sense, akin to adding an infinitely-weighted $\|\btG \bw - \bth\|^2$ term to the ILOP objective.

{\bf Inequality constraints affinely-dependent in $\bw$. \;} After transforming affinely-dependent inequality constraints to $\btA' \bw' \leq \btb'$, detecting redundancy among these constraints can be as hard as solving an LP \citep{telgen1983identifying}. Generally, inequality constraint $\ba_j^T \bw \leq b_j$ is redundant with respect to $\bA \bw \leq \bb$ if and only if the optimal value of the following LP is non-negative:
\begin{equation} \label{eq:redundantlp}
\begin{aligned}
\minimize_{\bw} \quad & b_j - \ba_j^T\bw \\
\subjto \quad & \bA_{\{j' \neq j\}} \bw \leq \bb_{\{j' \neq j\}}
\end{aligned}
\end{equation}
Here $\ba_j$ is the $j^\text{th}$ row of $\bA$ and $\bA_{\{j' \neq j\}}$ is all the rows of $\bA$ except the $j^\text{th}$. If the optimal value to \eqref{eq:redundantlp} is non-negative then it says ``we tried to violate the $j^\text{th}$ constraint, but the other constraints prevented it, and so the $j^\text{th}$ constraint must be redundant.'' However, \citet{telgen1983identifying} reviews much more efficient methods of identifying redundant linear inequality constraints, by analysis of basic basic variables in a simplex tableau. \citet{zhen2018adjustable} proposed a `redundant constraint identification' (RCI) procedure proposed by that is directly analogous to \eqref{eq:redundantlp} along with another heuristic RCI procedure.

{\bf Constraints polynomially-dependent in $\bw$. \;} Similar to the affinely-dependent case, when the coefficients of constraints $\bA(\bu, \bw)\bx \leq \bb(\bu, \bw)$ and $\bG(\bu, \bw)\bx \leq \bh(\bu, \bw)$ are polynomially-dependent on $\bw$, we can rewrite the constraints in terms of $\bw$. Redundancy among equality constraints of the resulting system can be simplified by computing a minimal Gr\"obner basis \citep{cox2013ideals}, for example by Buchberger's algorithm which is a generalization of Gaussian elimination; see the paper by \citet{lim2012groebner} for a review of Gr\"obner basis techniques applicable over a real field. Redundancy among inequality constraints for nonlinear programming has been studied \citep{caron2009redundancy,obuchowska1995minimal}. Simplifying polynomial systems of equalities and inequalities is a subject of semialgebraic geometry and involves generalizations of Fourier-Motzkin elimination. Details are beyond the scope of this manuscript.


\subsection*{Appendix C: Proofs of Theorem 1 and Corollary 1}\label{sec:appendixB}

\begin{proof}[Proof of Theorem \ref{thm:oegrad}]
The dual linear program associated with \eqref{eq:LP} is
\begin{equation} \label{eq:DP} \tag{DP}
\begin{aligned}
\maximize_{\blambda,\, \bnu} \quad & \bb^T \blambda + \bh^T \bnu \\
\subjto         \quad & \bA^T \blambda + \bG^T \bnu \:=\: \bc\\
                      & \blambda \:\leq\:    \bzero,
\end{aligned}
\end{equation}
where $\blambda \in \bbR^{M_1}_{\leq 0}, \bnu \in \bbR^{M_2}$ are the associated dual variables for the primal inequality and equality constraints, respectively. 

Since $\bx^*$ is optimal to~\eqref{eq:LP} and $\blambda^*, \bnu^*$ are optimal to~\eqref{eq:DP}, then $(\bx^*,\;\blambda^*,\;\bnu^*)$ satisfy the KKT conditions (written specialized to the particular LP form we use):
\begin{equation} \label{eq:KKT} \tag{KKT}
\begin{aligned}
\bA \bx &\le \bb \\
\bG \bx &= \bh \\
\bA^T \blambda + \bG^T \bnu  &= \bc \\
\blambda &\leq \bzero \\
\bD(\blambda)(\bA \bx - \bb) &= \bzero \\
\end{aligned}
\end{equation}
where $\bD(\blambda)$ is the diagonal matrix having $\blambda$ on the diagonal.
The first two constraints correspond to primal feasibility, the next two to dual feasibility and the last one specifies complementary slackness. From here forward it should be understood that $\bx, \blambda, \bnu$ satisfy KKT even when not emphasized by $*$.

As in the paper by \citet{Amos2017}, implicitly differentiating the equality constraints in \eqref{eq:KKT} gives

\begin{equation} \label{eq:DKKT} \tag{DKKT}
\begin{aligned}
\bG \d\bx &= \d \bh - \d \bG\bx\\
\bA^T \d\blambda + \bG^T \d \bnu &= \d \bc - \d \bA^T \blambda - \d \bG^T \bnu \\
\bD(\blambda)\bA \d \bx  + \bD(\bA \bx - \bb) \d\blambda &=  \bD(\blambda)(\d \bb - \d \bA \bx)
\end{aligned}
\end{equation}

where $\d \bc, \d \bA, \d \bb, \d \bG, \d \bh$ are parameter differentials and $\d \bx, \d \blambda, \d \bnu$ are solution differentials, all having the same dimensions as the variables they correspond to. Because \eqref{eq:KKT} is a second-order system, \eqref{eq:DKKT} is a system of linear equations. Because the system is linear, a partial derivative such as $\pd[x_j^*]{b_i}$ can be determined (if it exists) by setting $\d b_i = 1$ and all other parameter differentials to~$0$, then solving the system for solution differential $\d x_j$, as shown by \citet{Amos2017}.

We can assume \eqref{eq:KKT} is feasible in $\bx, \blambda, \bnu$. In each case of the main proof it will be important to characterize conditions under which~\eqref{eq:DKKT} is then feasible in $\d \bx$.
This is because, if \eqref{eq:DKKT} is feasible in at least $\d \bx$, then by substitution we have

\begin{equation} \label{eq:oegrad1}
\begin{aligned}
\bc^T \d \bx &= (\bA^T \blambda + \bG^T \bnu)^T \d \bx \\
&= \blambda^T \bA \d \bx + \bnu^T \bG \d \bx \\
&= \blambda^T(\d \bb - \d \bA \bx) + \bnu^T (\d \bh - \d \bG \bx)
\end{aligned}
\end{equation}

and this substitution is what gives the total derivatives their form. In \eqref{eq:oegrad1} the substitution $\blambda^T \bA \d \bx = \blambda^T (\d \bb - \d \bA \bx)$ holds because $\bx, \blambda$ feasible in \eqref{eq:KKT} implies $\lambda_i < 0 \Rightarrow \bA_i \bx - b_i = 0$ in \eqref{eq:DKKT}, where $\bA_i$ is the $i$\textsuperscript{th} row of $\bA$. Whenever $\d \bx$ is feasible in \eqref{eq:DKKT} we have $\lambda_i \bA_i \d \bx = \lambda_i(\d b_i - \d \bA_i \bx)$ for any $\lambda_i \leq 0$, where $\d \bA_i$ is the $i$\textsuperscript{th} row of differential $\d \bA$.

Note that \eqref{eq:oegrad1} holds even if $\text{(DKKT)}$ is not feasible in $\d \blambda$ and/or $\d \bnu$. In other words, it does not require the KKT point $(\bx^*, \blambda^*, \bnu^*)$ to be differentiable with respect to $\blambda^*$ and/or $\bnu^*$.

Given a KKT point $(\bx^*, \blambda^*, \bnu^*)$ let $\cI, \cJ, \cK$ be a partition of inequality indices $\{1, \ldots, M_1\}$ where
\begin{equation*}
\begin{aligned}
\cI &= \left\{\, i : \lambda^*_i < 0,\, \bA_i \bx^* = b_i \,\right\} \\
\cJ &= \left\{\, i : \lambda^*_i = 0,\, \bA_i \bx^* < b_i \,\right\} \\
\cK &= \left\{\, i : \lambda^*_i = 0,\, \bA_i \bx^* = b_i \,\right\} \\
\end{aligned}
\end{equation*}
and the corresponding submatrices of $\bA$ are $\bA_\cI, \bA_\cJ, \bA_\cK$. Then \eqref{eq:DKKT} in matrix form is
\begin{equation} \label{eq:oegrad2}
\begin{bmatrix}
\bG & \bzero & \bzero & \bzero & \bzero \\
\bD(\blambda_\cI)\bA_\cI & \bzero & \bzero & \bzero & \bzero\\
\bzero & \bzero & \bD(\bA_\cJ \bx - \bb_\cJ) & \bzero& \bzero  \\
\bzero & \bzero & \bzero & \bzero & \bzero  \\
\bzero & \bA^T_\cI & \bA^T_\cJ & \bA^T_\cK & \bG^T \\
\end{bmatrix}
\begin{bmatrix}
\,\d \bx_{\phantom{\cI}} \\
\,\d \blambda_\cI \\
\,\d \blambda_\cJ \\
\,\d \blambda_\cK \\
\,\d \bnu_{\phantom{\cI}} \\
\end{bmatrix} = 
\begin{bmatrix}
\d \bh - \d \bG \bx \\
\d \bb_\cI - \d \bA_\cI \bx \\
\bzero \\
\bzero \\
\: \d \bc - \d \bA^T \blambda - \d \bG^T \bnu \: \\
\end{bmatrix}
\end{equation}

The pattern of the proof in each case will be to characterize feasibility of \eqref{eq:oegrad2} in $\d \bx$ and then apply \eqref{eq:oegrad1} for the result.

{\bf Evaluating $\pd[z]{\bc}$. \;} Consider $\pd[z]{c_j} = x^\mathrm{obs}_j - x^*_j - \bc^T \pd[\bx^*]{c_j}$. To evaluate the $\bc^T \pd[\bx^*]{c_j}$ term, set $\d c_j = 1$ and all other parameter differentials to $0$. Then the right-hand side of~\eqref{eq:oegrad2} becomes

\begin{equation} \label{eq:oegrad2c}
\begin{bmatrix}
\bG & \bzero & \bzero & \bzero & \bzero \\
\bD(\blambda_\cI)\bA_\cI & \bzero & \bzero & \bzero & \bzero\\
\bzero & \bzero & \bD(\bA_\cJ \bx - \bb_\cJ) & \bzero& \bzero  \\
\bzero & \bzero & \bzero & \bzero & \bzero  \\
\bzero & \bA^T_\cI & \bA^T_\cJ & \bA^T_\cK & \bG^T \\
\end{bmatrix}
\begin{bmatrix}
\,\d \bx_{\phantom{\cI}} \\
\,\d \blambda_\cI \\
\,\d \blambda_\cJ \\
\,\d \blambda_\cK \\
\,\d \bnu_{\phantom{\cI}} \\
\end{bmatrix} = 
\begin{bmatrix}
\bzero \\
\bzero \\
\bzero \\
\bzero \\
\bone^j \\
\end{bmatrix}
\end{equation}

where $\bone^j$ denotes the vector with $1$ for component $j$ and $0$ elsewhere.
System~\eqref{eq:oegrad2c} is feasible in $\d \bx$ (not necessarily unique) so we can apply~\eqref{eq:oegrad1} to get $\bc^T \pd[\bx^*]{c_j} = \bc^T \d \bx = \blambda^T(\bzero - \bzero \bx) + \bnu^T(\bzero - \bzero \bx) = 0$. The result for $\pd[z]{\bc}$ then follows from $\bc^T \pd[\bx^*]{\bc} = \bzero$.

{\bf Evaluating $\pd[z]{\bh}$. \;} Consider $\pd[z]{h_i} = -\bc^T \pd[\bx^*]{h_i}$. Set $\d h_i = 1$ and all other parameter differentials to $0$. Then the right-hand side of~\eqref{eq:oegrad2} becomes

\begin{equation} \label{eq:oegrad2h}
\begin{bmatrix}
\bG & \bzero & \bzero & \bzero & \bzero \\
\bD(\blambda_\cI)\bA_\cI & \bzero & \bzero & \bzero & \bzero\\
\bzero & \bzero & \bD(\bA_\cJ \bx - \bb_\cJ) & \bzero& \bzero  \\
\bzero & \bzero & \bzero & \bzero & \bzero  \\
\bzero & \bA^T_\cI & \bA^T_\cJ & \bA^T_\cK & \bG^T \\
\end{bmatrix}
\begin{bmatrix}
\,\d \bx_{\phantom{\cI}} \\
\,\d \blambda_\cI \\
\,\d \blambda_\cJ \\
\,\d \blambda_\cK \\
\,\d \bnu_{\phantom{\cI}} \\
\end{bmatrix} = 
\begin{bmatrix}
\bone^i \\
\bzero \\
\bzero \\
\bzero \\
\bzero \\
\end{bmatrix}
\end{equation}

Since $\bx^*$ is non-degenerate in the sense of \citet{tijssen1998balinski}, then there are at most $D$ active constraints (including equality constraints) and the rows of $\begin{bmatrix} \bG \\ \bA_\cI \end{bmatrix}$ are also linearly independent. Since active constraints are linearly independent, system~\eqref{eq:oegrad2h} is feasible in $\d \bx$ across all $i \in \{1, \ldots, M_2\}$. We can therefore apply~\eqref{eq:oegrad1} to get $\bc^T \pd[\bx^*]{h_i} = \bc^T \d \bx = \blambda^T(\bzero - \bzero \bx) + \bnu^T(\bone^i - \bzero \bx) = \nu_i$. The result for $\pd[z]{\bh}$ then follows from $\bc^T \pd[\bx^*]{\bh} = \bnu^{*T}$.

{\bf Evaluating $\pd[z]{\bb}$. \;} Consider $\pd[z]{b_i} = -\bc^T \pd[\bx^*]{b_i}$. Set $\d b_i = 1$ and all other parameter differentials to $0$. For $i \in \cI$ the right-hand side of~\eqref{eq:oegrad2} becomes

\begin{equation} \label{eq:oegrad2b}
\begin{bmatrix}
\bG & \bzero & \bzero & \bzero & \bzero \\
\bD(\blambda_\cI)\bA_\cI & \bzero & \bzero & \bzero & \bzero\\
\bzero & \bzero & \bD(\bA_\cJ \bx - \bb_\cJ) & \bzero& \bzero  \\
\bzero & \bzero & \bzero & \bzero & \bzero  \\
\bzero & \bA^T_\cI & \bA^T_\cJ & \bA^T_\cK & \bG^T \\
\end{bmatrix}
\begin{bmatrix}
\,\d \bx_{\phantom{\cI}} \\
\,\d \blambda_\cI \\
\,\d \blambda_\cJ \\
\,\d \blambda_\cK \\
\,\d \bnu_{\phantom{\cI}} \\
\end{bmatrix} = 
\begin{bmatrix}
\bzero \\
\lambda_i \bone^i \\
\bzero \\
\bzero \\
\bzero \\
\end{bmatrix}
\end{equation}

Since $\bx^*$ is non-degenerate, then system~\eqref{eq:oegrad2b} is feasible in $\d \bx$ for all $i \in \cI$ by identical reasoning as for $\pd[z]{h_i}$. For $i \in \cJ \cup \cK$ the right-hand side of~\eqref{eq:oegrad2} is zero and so the system is feasible in $\d \bx$. System~\eqref{eq:oegrad2b} is therefore feasible in $\d \bx$ across all $i \in \{1, \ldots, M_1\}$. We can therefore apply~\eqref{eq:oegrad1} to get $\bc^T \pd[\bx^*]{b_i} = \bc^T \d \bx = \blambda^T(\bone^i - \bzero \bx) + \bnu^T(\bzero - \bzero \bx) = \lambda_i$. The result for $\pd[z]{\bb}$ then follows from $\bc^T \pd[\bx^*]{\bb} = \blambda^{*T}$.

{\bf Evaluating $\pd[z]{\bG}$. \;} Consider $\pd[z]{G_{ij}} = -\bc^T \pd[\bx^*]{G_{ij}}$. Set $\d G_{ij} = 1$ and all other parameter differentials to $0$. Then the right-hand side of~\eqref{eq:oegrad2} becomes

\begin{equation} \label{eq:oegrad2g}
\begin{bmatrix}
\bG & \bzero & \bzero & \bzero & \bzero \\
\bD(\blambda_\cI)\bA_\cI & \bzero & \bzero & \bzero & \bzero\\
\bzero & \bzero & \bD(\bA_\cJ \bx - \bb_\cJ) & \bzero& \bzero  \\
\bzero & \bzero & \bzero & \bzero & \bzero  \\
\bzero & \bA^T_\cI & \bA^T_\cJ & \bA^T_\cK & \bG^T \\
\end{bmatrix}
\begin{bmatrix}
\,\d \bx_{\phantom{\cI}} \\
\,\d \blambda_\cI \\
\,\d \blambda_\cJ \\
\,\d \blambda_\cK \\
\,\d \bnu_{\phantom{\cI}} \\
\end{bmatrix} = 
\begin{bmatrix}
-x_j \bone^{i} \\
\bzero \\
\bzero \\
\bzero \\
-\nu_i \bone^{j} \\
\end{bmatrix}
\end{equation}

Since $\bx^*$ is non-degenerate, then~\eqref{eq:oegrad2g} is feasible in $\d \bx$ for all $i \in \{1, \ldots, M_2\}$ and $j \in \{1,\ldots,D\}$ by same reasoning as $\pd[z]{\bh}$. Applying~\eqref{eq:oegrad1} gives $\bc^T \pd[\bx^*]{G_{ij}} = \bc^T \d \bx = \blambda^T(\bzero - \bzero \bx) + \bnu^T(\bzero - \bone^{ij} \bx) = -\nu_i x_j$ where $\bone^{ij}$ is the $M_2 \times D$ matrix with $1$ for component $(i,j)$ and zeros elsewhere. The result for $\pd[z]{\bG}$ then follows from $\bc^T \pd[\bx^*]{\bG} = -\bnu^{*}\bx^{*T}$ where we have slightly abused notation by dropping the leading singleton dimension of the $1 \times M_2 \times D$ Jacobian.

{\bf Evaluating $\pd[z]{\bA}$. \;} Consider $\pd[z]{A_{ij}} = -\bc^T \pd[\bx^*]{A_{ij}}$. Set $\d A_{ij} = 1$ and all other parameter differentials to $0$. Then the right-hand side of~\eqref{eq:oegrad2} becomes

\begin{equation} \label{eq:oegrad2a}
\begin{bmatrix}
\bG & \bzero & \bzero & \bzero & \bzero \\
\bD(\blambda_\cI)\bA_\cI & \bzero & \bzero & \bzero & \bzero\\
\bzero & \bzero & \bD(\bA_\cJ \bx - \bb_\cJ) & \bzero& \bzero  \\
\bzero & \bzero & \bzero & \bzero & \bzero  \\
\bzero & \bA^T_\cI & \bA^T_\cJ & \bA^T_\cK & \bG^T \\
\end{bmatrix}
\begin{bmatrix}
\,\d \bx_{\phantom{\cI}} \\
\,\d \blambda_\cI \\
\,\d \blambda_\cJ \\
\,\d \blambda_\cK \\
\,\d \bnu_{\phantom{\cI}} \\
\end{bmatrix} = 
\begin{bmatrix}
\bzero \\
-x_j \bone^{i} \\
\bzero \\
\bzero \\
-\lambda_i \bone^{j} \\
\end{bmatrix}
\end{equation}

Since $\bx^*$ is non-degenerate, then by similar arguments as $\pd[z]{\bb}$ and $\pd[z]{\bG}$ \eqref{eq:oegrad2a} is feasible in $\d \bx$ for all $i \in \{1, \ldots, M_1\}$ and $j \in \{1, \ldots, D\}$ and the result for $\pd[z]{\bA}$ follows from $\bc^T \pd[\bx^*]{\bG} = -\blambda^{*}\bx^{*T}$.
\end{proof}

\begin{proof}[Proof of Corollary \ref{thm:oeunique}]
The result for $\pd[z]{\bc}$ is direct. In linear programming, \citet{tijssen1998balinski} showed that the existence of a non-degenerate primal solution $\bx^*$ implies uniqueness of the dual solution $\blambda^*, \bnu^*$ so the result for $\pd[z]{\bb}$ and $\pd[z]{\bh}$ follows directly.
If a non-degenerate solution $\bx^*$ is unique then matrices $\blambda^* \bx^{*T}$ and $\bnu^* \bx^{*T}$ are both unique, regardless of whether $\bc=\bzero$. In the other direction, if $\blambda^* \bx^{*T}$ and $\bnu^* \bx^{*T}$ are both unique, consider two mutually exclusive and exhaustive cases: (1) when either $\blambda^* \neq \bzero$ or $\bnu^* \neq \bzero$ this would imply $\bx^*$ unique, and (2) when both $\blambda^* = \bzero$ and $\bnu^* = \bzero$ in \eqref{eq:DP} this would imply $\bc = \bzero$, {\em i.e.} the primal linear program \eqref{eq:LP} is merely a feasibility problem. The result for $\pd[z]{\bA}$ and $\pd[z]{\bG}$ then follows.
\end{proof}

\subsection*{Appendix D: Additional Results}\label{sec:appendixC}

Figure~\ref{fig:1a} shows the task of learning ($\bc$, $\bA$, $\bb$) with a $K\!=\!6$ dimensional parametrization $\bw$ and 20 training observations for a $D$ dimensional decision space $\bx$ with $M_1$ inequality constraints. The five different considered combinations of $D$ and $M_1$ are shown in the figure. The results over all problem sizes are similar to the case of $D\!=\!10, M_1\!=\!80$ shown in the main paper. RS fails; COBYLA `succeeds' on ~25\% of instances; SQP succeeds on 60-75\%, which is substantially better. 
As expected, instances with higher $D$, are more challenging as we observe that the success rate decreases slightly.
The success curve of \SQPbprop slightly lags those of \SQPimpl and \SQPdir due to the overhead of backpropagating through the steps of the interior point solver. 
However, this computational advantage of \SQPimpl and \SQPdir over \SQPbprop is less obvious on LP instances with $D=10$. For larger LP instances, the overall framework spends significantly more computation time on other components (e.g., solving the forward problem, solving (SQP)). Thus, the advantage of \SQPimpl and \SQPdir in computing gradients is less significant in the overall performance.

We observe similar performance on instances with equality constraints, where $\bG$ and $\bh$ also need to be learned; see Figure~\ref{fig:1b}. Note that RS failed to find a feasible $\bw$ in all instances, caused mainly by the failure to satisfy the equality target feasibility constraints in \eqref{eqn:IOP_outer}. Recall that a feasible $\bw$ means both \eqref{eqn:IOP_outer} and \eqref{eqn:IOP_inner} are satisfied. 

Figure~\ref{fig:1c} shows the performance on the LPs, where the dimensionality of $\bw$ is higher. We observe that COBYLA performs poorly, while SQP methods succeed on all instances. This is caused by the finite-difference approximation technique used in COBYLA which is inefficient in high dimension $\bw$ space. This result demonstrates the importance of using gradient-based methods in high dimensional (in $\bw$) NLP. 

\renewcommand{\thesubfigure}{\roman{subfigure}}

\begin{figure}
\centering
    \begin{subfigure}{\textwidth}
      \centering
          \includegraphics[width=\linewidth]{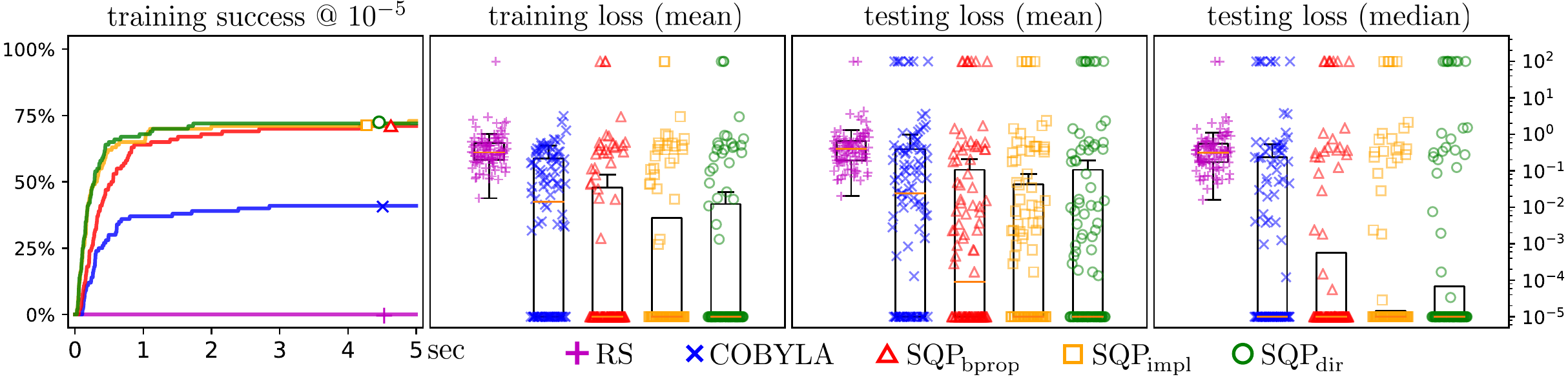}\vspace{-.4em}
          \caption{$D\!=\!2, M_1\!=\!4$\vspace{.4em}}
          \label{D2M14}
    \end{subfigure} 
    \begin{subfigure}{\textwidth}
      \centering
          \includegraphics[width=\linewidth]{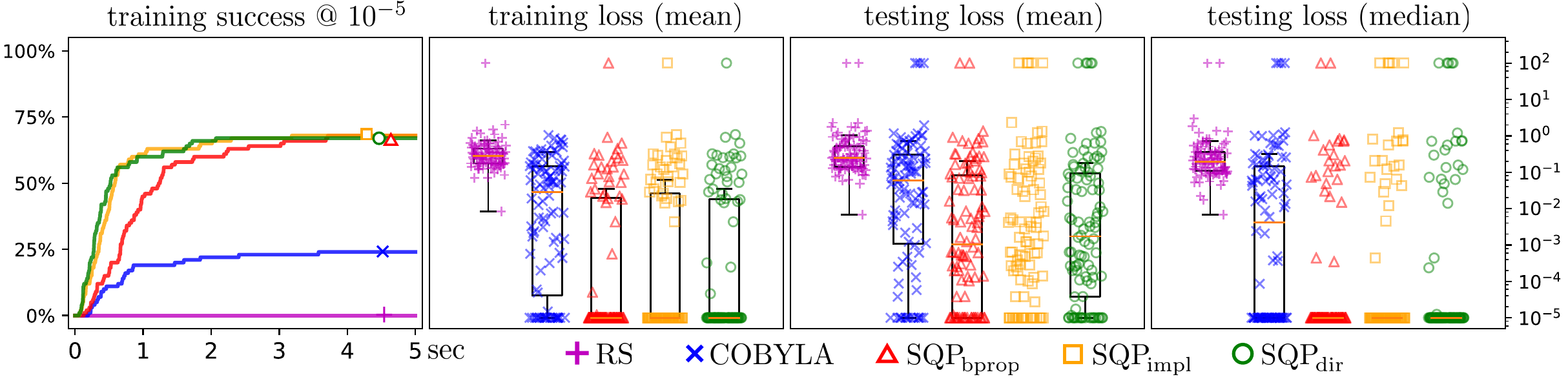}\vspace{-.4em}
          \caption{$D\!=\!2, M_1\!=\!8$\vspace{.4em}}
          \label{D2M18}
    \end{subfigure} \\
    \begin{subfigure}{\textwidth}
      \centering
          \includegraphics[width=\linewidth]{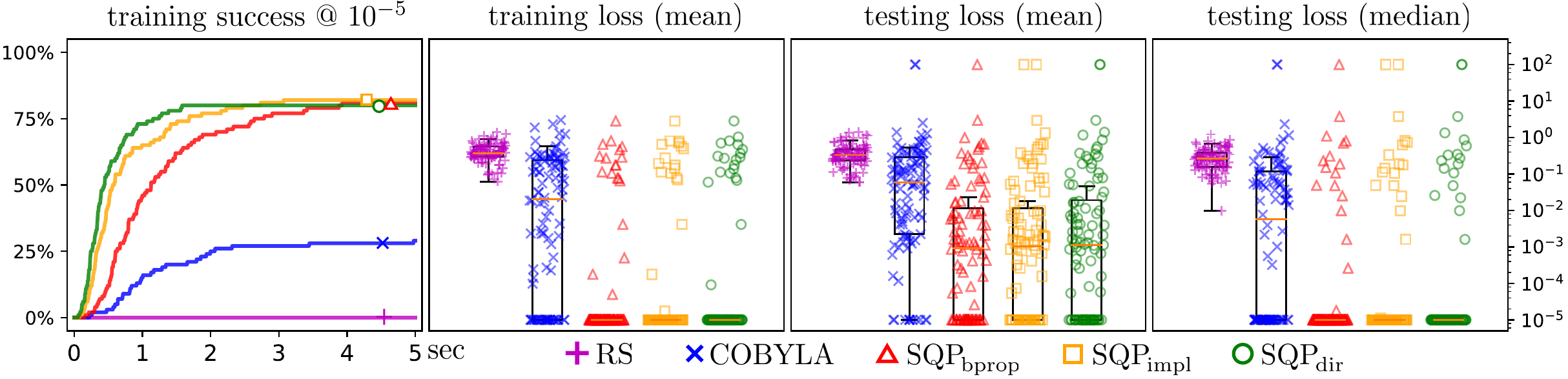}\vspace{-.4em}
          \caption{$D\!=\!2, M_1\!=\!16$\vspace{.4em}}
          \label{D2M116}
    \end{subfigure}
    \begin{subfigure}{\textwidth}
      \centering
          \includegraphics[width=\linewidth]{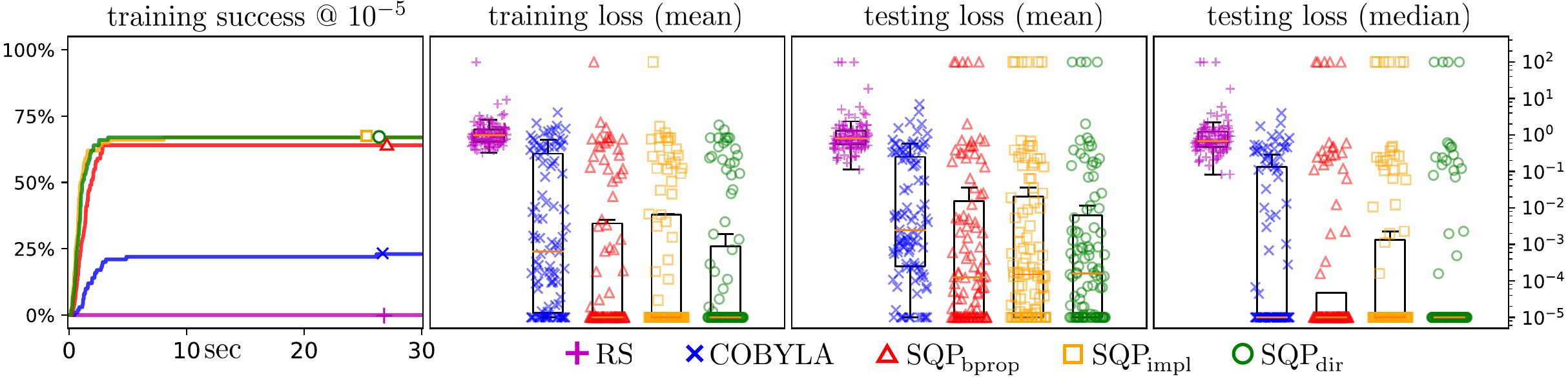}\vspace{-.4em}
          \caption{$D\!=\!10, M_1\!=\!20$\vspace{.4em}}
          \label{D10M120}
    \end{subfigure}
    \begin{subfigure}{\textwidth}
      \centering
          \includegraphics[width=\linewidth]{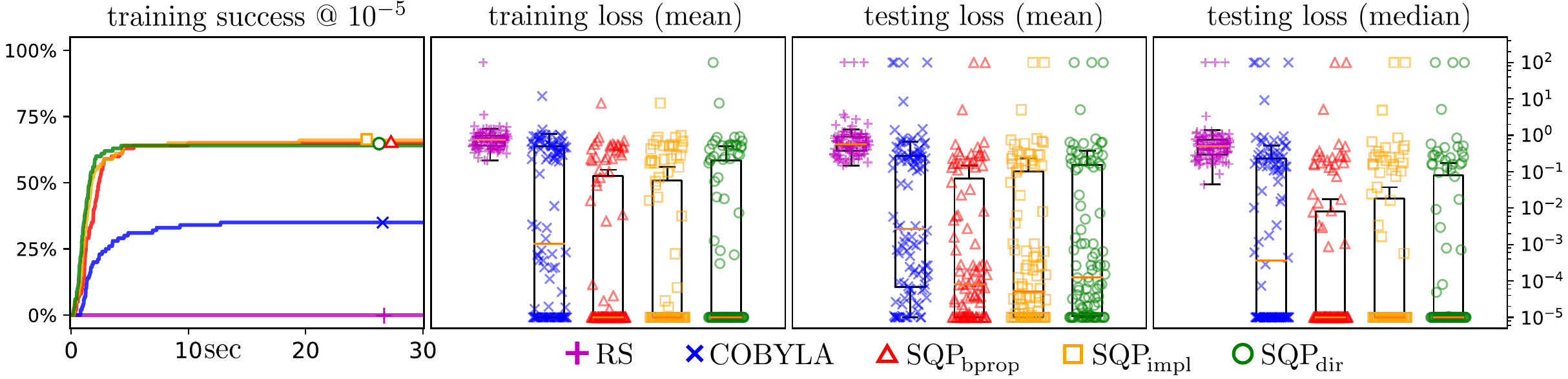}\vspace{-.4em}
          \caption{$D\!=\!10, M_1\!=\!36$\vspace{.4em}}
          \label{D10M136}
    \end{subfigure}    

\caption{A comparison on synthetic PLP instances with varying $D$ and $M_1$. Shown is the probability of achieving zero AOE training loss over time (curves), along with final training and testing loss (box plots). Each mark denotes one of 100 trials (different instances) with 20 training and 20 testing points (problem sizes are indicated for each sub-figure). The AOE testing loss is always evaluated with the `true' cost $\bc$, never the imputed cost. For insight into why the mean testing error is larger than median testing error, see discussion (end of Section~\ref{sec:experiments}).}
\label{fig:1a}
\end{figure}

\begin{figure}[h]
\centering
      \centering
          \includegraphics[width=\textwidth]{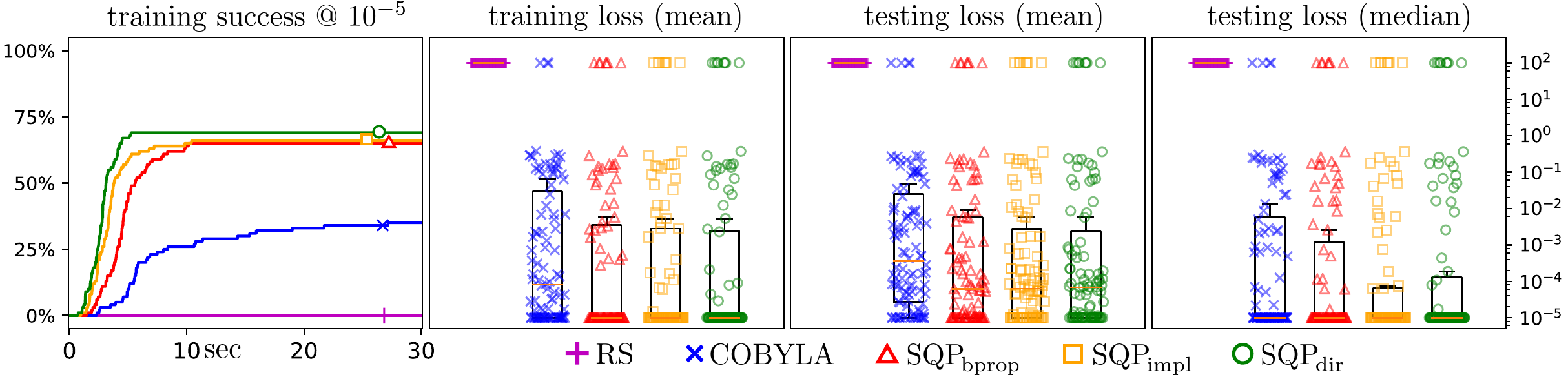}
          \caption{A comparison on synthetic PLP instances with equality constraints ($D\;=\;10$, $M_1\;=\;80$, $M_2\;=\;2$.). Shown is the probability of achieving zero AOE training loss over time (curves), along with final training and testing loss (box plots). Each mark denotes one of 100 trials (different instances) with 20 training and 20 testing points. The AOE testing loss is always evaluated with the `true' cost $\bc$, never the imputed cost.\label{fig:1b}}
\end{figure}

\begin{figure}
\centering
        \begin{subfigure}{\textwidth}
          \centering
              \includegraphics[width=0.67\textwidth]{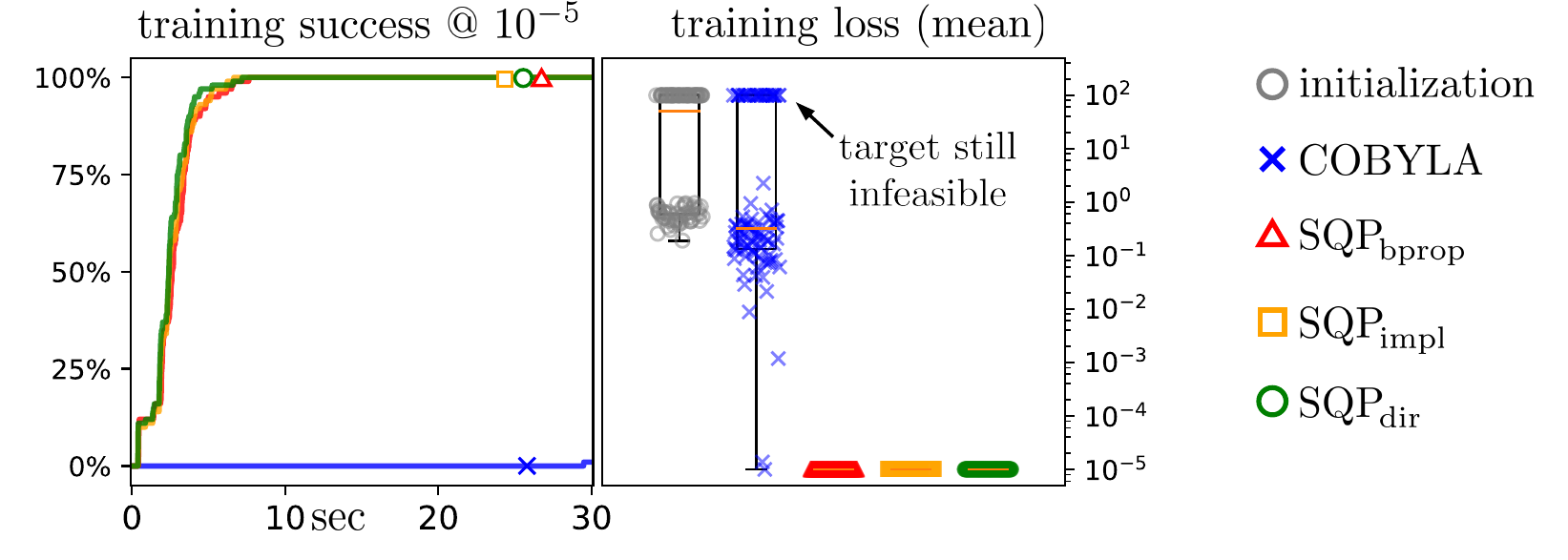}
        \end{subfigure} 
\caption{A comparison on synthetic LP instances  ($D\;=\;10$, $M_1\;=\;80$). Shown is the probability of achieving zero AOE training loss over time (curves), along with final loss (box plots). Each mark denotes one of 100 trials (different instances), each with one training point. Note, in this experiment we aim to learn LP coefficients directly, i.e., $\bw$ comprises all LP coefficients, and the LP coefficients do not depend on $\bu$. Therefore, there is only a single target solution for learning $\bw$, and no testing data.
\label{fig:1c}}
\end{figure}

\textbf{Sensitivity of results to parameter settings\;} The specific results of our experiments can vary slightly with certain choices, but the larger conclusions do not change: the gradient-based SQP methods all perform similarly, and they consistently out-perform non-gradient-based methods, especially for higher-dimensional search. 

Specific choices of parameter settings include numerical tolerance used in the forward solve (e.g. $10^{-5}$ vs $10^{-8}$), algorithm terminate tolerance of the COBYLA and SLSQP, and even PyTorch version (v1.5 vs. nightly builds). 
For example, we tried using strict tolerances and different trust region sizes for COBYLA to encourage the algorithm to search more aggressively, but these made only a small improvement to performance; these small improvements are represented in our results. 
We also observed that, although the homogeneous solver works slightly better when we use a strict numerical tolerance, there is no major difference in the learning results.

In conclusion, our main experiment results are largely insensitive to specific parameter settings.

\subsection*{Appendix E: Parametric Linear Program for Figure~\ref{fig:decisionsurface}}\label{sec:appendixD}

{\em Forward optimization problem for Figure~\ref{fig:decisionsurface}.}
The FOP formulation used is shown in~\eqref{eq:decisionsurface_plp} below. 

\begin{equation} \label{eq:decisionsurface_plp}
\begin{alignedat}{1}
\minimize_{x_1, x_2} \quad & -w_1 u_1 x_1 - w_2 u_2 x_2 \\
\subjto   \quad & x_1 + x_2  \leq \max(1, u_1 + u_2) \\
                & 0 \leq x_1 \leq 1 \\
                & 0 \leq x_2 \leq 1 \\
\end{alignedat}
\end{equation}

The two training points are generated with $\bw=(1, 1)$ at $\bu_1=(1, \frac{1}{3})$ and $\bu_2=(1, \frac{1}{3})$ with testing point $\bu_\mathrm{test}=(\frac{1}{2}, \frac{5}{6})$. PLP learning was initialized at $\bw_\mathrm{ini}=(4, 1)$ and the \SQPimpl algorithm returned $\bw_\mathrm{learned}\approx(\frac{35}{9},\frac{4}{3})$, used to generate the learned decision map depicted in the figure.


\end{document}